\documentclass{article}
\usepackage{iclr2027_conference,times}


\usepackage{amsmath,amsfonts,bm}

\def\eqref#1{equation~\ref{#1}}

\def\1{\bm{1}}

\DeclareMathAlphabet{\mathsfit}{\encodingdefault}{\sfdefault}{m}{sl}
\SetMathAlphabet{\mathsfit}{bold}{\encodingdefault}{\sfdefault}{bx}{n}

\newcommand{\E}{\mathbb{E}}

\newcommand{\R}{\mathbb{R}}

\DeclareMathOperator*{\argmax}{arg\,max}

\usepackage[utf8]{inputenc}
\usepackage[T1]{fontenc}
\usepackage{url}
\usepackage{booktabs}
\usepackage{microtype}
\usepackage{xcolor}
\usepackage{amsmath,amssymb,amsthm}
\usepackage{mathtools}
\usepackage{graphicx}
\usepackage{multirow}
\usepackage{placeins}
\usepackage{etoolbox}
\usepackage{array}
\usepackage{enumitem}
\usepackage{algorithm}
\usepackage{algpseudocode}
\usepackage{hyperref}
\graphicspath{{figures/}}

\renewcommand{\R}{\mathbb{R}}
\renewcommand{\E}{\mathbb{E}}

\newcommand{\pishape}{\pi_{\text{shaped}}}
\newcommand{\pimu}{\pi_\mu}
\newcommand{\Qpi}{Q^{\pimu}}
\newcommand{\dpimu}{d^{\pimu}}

\newcommand{\cA}{\mathcal{A}}
\newcommand{\sg}{\operatorname{sg}}

\newtheorem{theorem}{Theorem}
\newtheorem{corollary}[theorem]{Corollary}
\newtheorem{proposition}[theorem]{Proposition}
\newtheorem{lemma}[theorem]{Lemma}
\newtheorem{assumption}[theorem]{Assumption}

\title{Action Shaping:\texorpdfstring{\\}{ }Policies Absorb What They Can Express}

\author{%
  Yanjun Chen$^{1,2}$\thanks{Correspondence: \texttt{yan-jun.chen@connect.polyu.hk}}
  \And
  Jinghan Wang$^{3}$
  \And
  Xiaoyu Shen$^{1}$
  \And
  Wenjie Li$^{2}$
  \And
  Wei Zhang$^{1}$\thanks{Corresponding author: \texttt{zhw@eitech.edu.cn}}
  \AND
  \mdseries
  $^{1}$Eastern Institute of Technology \quad
  $^{2}$The Hong Kong Polytechnic University \\
  \mdseries
  $^{3}$Harbin Institute of Technology
}

\iclrfinalcopy

\begin{document}
\maketitle
\lhead{Preprint. Under review.}

\begin{abstract}
Reward shaping has a theorem: a potential-based term can be removed without changing the optimal policy.
The same practice on the action channel, an offset added in training and dropped at deployment, has no theorem.
Nothing cancels an action offset, so the correction is kept at deployment or removed without a guarantee.
We call it \emph{action shaping} and state its principle.
A trainable policy absorbs an offset its own output layer can reproduce exactly, which is what we mean by \emph{express}; what is absorbed can be removed with the return intact.
Its minimal instance is a zero-initialized linear head behind a learnable gate, added to an actor that trains through a learned action-value function, with no penalty or schedule.
The gate rises and then falls on its own, for deterministic and stochastic actors alike, and on 20 tasks removing the head costs almost nothing.
With default hyperparameters and no per-task tuning, the head also raises TD3's interquartile mean by 0.077 on the same 20 tasks and that of SAC with an ensemble critic at all three update ratios tested, and on 26 Atari games the actor deployed without it outscores its host.
The condition is exact reproduction, not capacity: a nonlinear head with more parameters is not absorbed, and in a paired control, one linear path added to a nonlinear base head restores absorption.
Exact reproduction gives the loss a flat direction that gradient noise drifts along, and the offset's amplitude indicates, before removal, what dropping the head will cost.
Action shaping thus gains the counterpart of the shaping theorem, a condition for absorption, together with the mechanism behind it and a diagnostic that reads it.
Policies absorb what they can express, and only that.
\end{abstract}

\section{Introduction}
\label{sec:intro}

Reinforcement learning routinely trains with structure it does not deploy: shaped rewards, privileged observations, exploration bonuses, expert corrections.
For the reward channel there is a theorem that says which of these can be removed without changing the optimal policy, and it has guided reward shaping~\citep{ng1999reward}.
A privileged critic can be removed for free by construction, since the deployed policy never computes it; the action channel has nothing comparable.
An offset added to the action changes which states are visited, no telescoping sum cancels it, and a co-trained correction is kept at deployment or removed without a guarantee.
We call this practice \emph{action shaping} and ask the reward-shaping question of it: when can the offset be removed for free? We show that the answer is not an invariance but an optimization phenomenon (Figure~\ref{fig:concept}).

\begin{figure}[t]
  \centering
  \includegraphics[width=\textwidth]{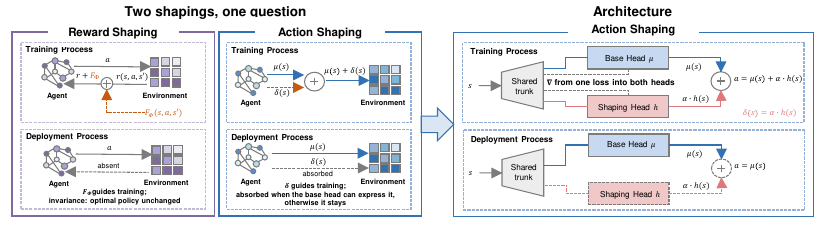}
  \caption{\textbf{A trainable policy absorbs a training-time action offset its own output layer can reproduce, and only that.} Reward shaping adds a term in training and a theorem says which terms leave the optimal policy unchanged; action shaping adds an offset $\delta = \alpha h$ to the action, and nothing cancels it. What removes it for free is absorption: when the base head $\mu$ can reproduce $\mu + \alpha h$ with its own weights, $W_\mu + \alpha W_h$, and is free to train, the gate $\alpha$ falls back toward zero and removing the head costs almost nothing. On a frozen base, or for a head that $\mu$ cannot reproduce, it is not absorbed. Right: the shaped actor in training and the base actor alone at deployment.}
  \label{fig:concept}
\end{figure}

A head added to a TD3 or SAC actor for training is \emph{absorbed}: its scalar gate, initialized near zero, rises early and then, with no schedule, no penalty, and no weight decay, falls back toward zero, and removing the head afterwards costs almost nothing.
Training with the head also raises TD3's interquartile mean over 20 environments by 0.077, with the default hyperparameters and no per-task tuning (\S\ref{sec:benefit}).
The head is the minimal instance of action shaping, a second linear layer on the representation the actor already computes, so its absorbed state is known in closed form, an exact reference point for measuring what training does.
With a one-hidden-layer network as the head, the gate does not return to zero, and removing the head at the end costs 0.44 to 0.78 of normalized return.
With the linear head's gradient to the shared trunk cut, so that the representation is shaped by the base head alone, the gate returns to zero as before.
Whether the head is absorbed is decided not by whether it sends gradient into the trunk but by what the base can do with its own weights.

The relevant sense of ``can'' is precise and central to this paper: a trainable trunk can approximate a nonlinear offset, so capacity does not separate the two heads.
What separates them is that for the linear head there is a change of the base head's own weights, and of nothing else, that reproduces the offset exactly, and for the nonlinear head there is none.
We call an offset \emph{expressible} when such a change exists: a linear offset on the representation the base head reads is expressible, a nonlinear offset is not, and neither is a linear offset on a nonlinear base head.
Under expressibility the shaped actor and a base-only actor with merged readout weights compute the same function, so the loss is constant along a line in parameter space that transfers the offset to the base, and we prove that, at a minimizer, the point on that line where the head contributes nothing is the flattest (\S\ref{sec:theory}).
Gradient noise moves parameters along flat directions, and in the idealization of stochastic gradient descent with label noise we derive a closed-form prediction: the gate decays exponentially, and the head's weights decay too, but more slowly, by a factor of the head's size. The agents, trained with Adam and a learned critic, show that pattern.
Without expressibility there is no line; without a trainable base there is no motion along it.
Policies absorb what they can express, in this sense of the word.

The condition sorts existing practices.
A training-time offset the base can express and is free to learn is absorbed; on a frozen base it stays, as in residual policy learning; one the base cannot express stays too, or collapses under strong noise, or is removed on a schedule, as in decaying action priors and attenuated residual policies.
The offset's amplitude falls as the head is absorbed, so the training log alone tells what removal will cost.
Across 20 environments, those whose amplitude ends below a tenth of its peak lose nothing in interquartile mean when the head is removed, and the others lose 0.03 of normalized return; within one run, removal is costly around the peak and small at each end.

We make four contributions.
\textbf{Absorption requires expressibility and a trainable base.} A zero-initialized, scalar-gated linear head co-trained with a TD3 or SAC actor engages and is then absorbed; seven arms on a six-task suite show that absorption requires an expressible offset and a trainable base, does not require gradient flow from the head into the trunk, and, in a paired control, returns on a nonlinear base head as soon as one linear path makes the offset reproducible again (\S\ref{sec:mechanism}).
\textbf{Expressibility creates a flat line.} A flat-line proposition shows that an expressible offset gives the loss a flat direction that leads to a closed gate. Under label-noise stochastic gradient descent the drift along it closes the gate exponentially while the head decays more slowly, by a factor of its size, as a minimal model confirms (\S\ref{sec:theory}, Appendix~\ref{app:toy}).
\textbf{The amplitude predicts the removal cost.} Its terminal residual does so across tasks and its trajectory within a run, without a task-specific threshold (\S\ref{sec:measure}, \S\ref{sec:removal}).
\textbf{The head raises the return of TD3, ensemble-critic SAC, and discrete SAC.} With default hyperparameters and no per-task tuning, the shaped TD3 actor scores 0.077 higher in interquartile mean than TD3 over 20 environments, the shaped SAC actor with an ensemble critic is above its host at every update ratio tested, and on 26 Atari games the actor deployed without the head outscores the unshaped host (\S\ref{sec:benefit}, \S\ref{sec:discrete}).
A first-order improvement theorem also explains what moves the gate, which a state-value critic supplies only by sampling (\S\ref{sec:first-order}).
Anyone who drops at inference a training-time head or correction branch that acts on the output assumes that the base has absorbed it; this paper gives a condition under which that holds on the action channel, and a way to check it.

\section{Action Shaping and the amplitude diagnostic}
\label{sec:method}

\subsection{The shaped actor}
\label{sec:shaped-actor}

Let $z(s) = \phi(s) \in \R^k$ be the representation computed by the actor's trunk (two hidden layers in continuous control), and let the base head read an action from it, $\mu(s) = W_\mu z(s) + b_\mu \in \R^m$.
Our instance of action shaping adds a second head on the same representation, $h(s) = W_h z(s) + b_h$, and a single learnable scalar \emph{gate} $\alpha$, and acts with
\begin{equation}\label{eq:shaped}
  \pishape(s) \;=\; f\!\bigl(\mu(s) + \alpha\, h(s)\bigr),
  \qquad \delta(s) := \alpha\, h(s),
\end{equation}
where $f$ is the host algorithm's output map (a $\tanh$ scaled to the action bounds for TD3, the squashed Gaussian mean for SAC, a softmax over logits for discrete actions).
We call $\delta$ the \emph{offset} and $\|\delta\|$ its \emph{amplitude}.
The \emph{shaping head} is zero-initialized ($W_h = 0$, $b_h = 0$) and $\alpha \sim U(-\varepsilon, \varepsilon)$ with $\varepsilon = 0.01$, so the shaped actor coincides with the base actor at the start of training, while $\alpha \neq 0$ keeps the head's gradient nonzero;
$\varepsilon$ is the only added constant (sensitivity in Appendix~\ref{app:eps}).

Training changes nothing in the host algorithm except the action it differentiates through: the actor loss of TD3, SAC, or PPO is evaluated at the shaped action, and its gradient updates $\phi$, $W_\mu$, $W_h$, and $\alpha$ jointly with the same optimizer.
There is no penalty on $\alpha$ or $\delta$, no decay schedule, and no weight decay: the amplitude evolves under the actor loss alone.
At deployment the shaping head is dropped and the agent acts with $f(\mu(s))$ alone; the head adds one linear layer, under six thousand parameters in every continuous-control environment.

\subsection{What we measure}
\label{sec:measure}

We log two quantities throughout training: the amplitude $\|\delta_t\|$, the mean of $\|\delta(s)\|$ over a fixed batch of evaluation states at step $t$, and the \emph{removal cost}, the return of the shaped actor minus the return of the base actor alone, both returns evaluated every 5K steps on the same environment seeds and normalized per environment (Appendix~\ref{app:eval}).

The \emph{terminal residual} summarizes an amplitude trajectory in one number,
\begin{equation}\label{eq:residual}
  r \;=\; \frac{\operatorname{mean}_{t \in \mathcal{T}} \|\delta_t\|}{\max_t \|\delta_t\|},
\end{equation}
the fraction of its own peak that the offset retains at the end of training, where $\mathcal{T}$ is the set of the last 20 evaluations.
In continuous control $r$ is a ratio of pre-activation amplitudes; Appendix~\ref{app:p4} gives the action-space version, the reading to use where the output map saturates.
The residual $r$ is computed per run, and an environment's residual is the mean of $r$ over its seeds.
A run is \emph{non-engaged} when the peak amplitude never exceeds $10^{-3}$, so that the offset never takes part in training;
it is \emph{completed} when $r \le 0.10$, and \emph{incomplete} otherwise.
The threshold 0.10 is a fixed constant; the rank correlations of \S\ref{sec:removal} do not use it and stay between 0.565 and 0.671 ($p < 0.01$) under all four residual conventions of Appendix~\ref{app:eval}.

\paragraph{Reading the gate.}
The three states give a protocol for any training-time branch behind a learnable scalar: if the scalar returned to zero, the base alone computes what the shaped actor computed; if it rose and stayed, the branch still contributes and dropping it can cost return; if it never rose, the branch never took part.
The gate and the offset it scales, not the norm of the head, whose weights decay far more slowly (\S\ref{sec:flat-line}), are the quantities to read.

\subsection{Seven arms}
\label{sec:arms}

We train the same actor in seven experimental arms that differ only in what the shaping head reads, what the base is allowed to learn, and what either head can compute.
\begin{itemize}
  \item \textbf{Shared} (default): $h$ reads $z(s)$; every parameter trains.
  \item \textbf{Detach}: $h$ reads $\sg(z(s))$; the forward pass is identical to Shared, but no gradient flows from the shaping head into the trunk.
  \item \textbf{Independent}: $h$ reads its own trainable trunk $\phi'(s)$ of the same size; the two heads share no representation.
  \item \textbf{Frozen}: $\phi$ and $W_\mu$ are frozen at initialization; only $W_h$, $b_h$, and $\alpha$ train.
  \item \textbf{MLP head}: $h$ is a one-hidden-layer network on $z(s)$ with 16, 64, or 256 hidden units and a zero-initialized output layer; everything else is as in Shared.
  \item \textbf{MLP base head}: $\mu$ is a one-hidden-layer network on $z(s)$ with 64 hidden units and $h$ stays linear, so the offset is linear and the base head is not; a second version adds a zero-initialized linear map of $z(s)$ to the base head's output, which puts $\mu + \alpha h$ back inside the base head's own function class.
\end{itemize}
The first four arms vary how the base can change and whether gradients cross between the heads; the fifth varies what the offset can be and the last two what the base head can be.
We say the offset is \emph{expressible} when it lies in the function class of the base head on the shared representation, that is, when there exist base-head parameters that reproduce $\mu + \alpha h$ exactly for the current trunk.
A linear head on a linear base head is expressible by construction, $\mu(s) + \alpha h(s) = (W_\mu + \alpha W_h)\, z(s) + (b_\mu + \alpha b_h)$; an MLP head is not, since no base-head parameters reproduce a nonlinear function of $z$, and neither is a linear head on an MLP base head, whose function class is not closed under adding an affine map, unless that base head carries a linear path of its own.
Expressibility is relative to the base head: the trunk could in principle be retrained to make any offset linear in a new representation, and \S\ref{sec:mechanism} reports whether training does so.
Expressibility does not require the offset to be additive.

For the linear head the absorbed state is known in closed form, the base-only actor with weights $W_\mu + \alpha W_h$, and the offset $\delta$ is exactly that state's base head minus the current one; this reference point is what makes the measurement exact.
The question is whether training reaches it by itself, so that the head can be dropped.
The linear head is a probe of whether training absorbs a branch, and the arms of \S\ref{sec:mechanism} show that training absorbs exactly the expressible ones on a trainable base.

\section{Why the offset helps and why it should decay}
\label{sec:theory}

\subsection{First-order improvement and the critic}
\label{sec:first-order}

Let $G(\alpha)$ be the return of the shaped policy $\pi_\alpha$ as a function of $\alpha$, with $\pi_0 = \pimu$, and define the per-state functional
\begin{equation}\label{eq:functional}
  \Phi_h(s) \;:=\; \frac{d}{d\alpha}\, \E_{a \sim \pi_\alpha(\cdot \mid s)}\!\bigl[\Qpi(s,a)\bigr]\Big|_{\alpha = 0},
\end{equation}
which equals $\nabla_a \Qpi(s, f(\mu))^{\top} J_f(\mu)\, h(s)$ for a deterministic continuous policy, $J_f$ the Jacobian of the output map, and $\operatorname{Cov}_{\pimu(\cdot \mid s)}\bigl(h(s), \Qpi(s, \cdot)\bigr)$ for a softmax policy over discrete actions (Appendix~\ref{app:theorem1}).

\begin{theorem}[First-order improvement]\label{thm:first-order}
Under the regularity conditions of Appendix~\ref{app:theorem1},
$G(\alpha) = G(0) + \frac{\alpha}{1-\gamma}\, \E_{s \sim \dpimu}[\Phi_h(s)] + O(\alpha^2)$.
\end{theorem}

The offset helps to first order exactly when it is aligned, on average, with the critic's action gradient; for a deterministic policy the first-order term is the deterministic policy gradient~\citep{silver2014dpg} taken with respect to $\alpha$ at $\alpha = 0$. Two consequences follow.
First, the gradient of the actor loss with respect to $\alpha$ is the same inner product at the shaped action: a $Q$-critic supplies it as a pathwise gradient, whereas under PPO the actor loss carries it only through a score-function estimate over sampled actions (\S\ref{sec:benefit}).
Second, the first-order term vanishes as the base policy approaches optimality (Appendix~\ref{app:corollary}), so the gradient that raised $\alpha$ fades as the base improves; this explains why the amplitude stops rising, and \S\ref{sec:flat-line} why it decreases.

\subsection{Expressibility creates a flat line}
\label{sec:flat-line}

Write the actor's parameters as $\theta = (\phi, W_\mu, b_\mu, W_h, b_h, \alpha)$ and let $\mathcal{L}(\theta)$ be the actor loss of the host algorithm evaluated at the shaped action.
For a linear head, the shaped pre-activation is $(W_\mu + \alpha W_h) z + (b_\mu + \alpha b_h)$, so the loss depends on the head parameters only through the merged readout $\bar W = W_\mu + \alpha W_h$ and $\bar b = b_\mu + \alpha b_h$.

\begin{proposition}[Flat line; proof in Appendix~\ref{app:flat}]\label{prop:flat}
Let the head be linear.
For every $\theta$ and every $\lambda \in \R$, the parameter $\theta_\lambda = (\phi,\, W_\mu - \lambda W_h,\, b_\mu - \lambda b_h,\, W_h,\, b_h,\, \alpha + \lambda)$ satisfies $\mathcal{L}(\theta_\lambda) = \mathcal{L}(\theta)$.
Along this line the trace of the Hessian of $\mathcal{L}$ with respect to $(W_h, b_h)$ equals $(\alpha + \lambda)^2$ times a quantity independent of $\lambda$, and at a minimizer of $\mathcal{L}$ it is minimized at $\alpha + \lambda = 0$.
\end{proposition}

The line of Proposition~\ref{prop:flat} is the set of ways to split one and the same shaped actor between its base head and its shaping head.
At $\alpha = 0$ the base head alone computes the shaped actor's output and the shaping head contributes nothing; at a minimizer, this split is the flattest.
The loss is constant along the line, so a noiseless gradient flow does not move along it; what moves a parameter along a flat direction is gradient noise.
Near a manifold of minimizers, stochastic gradient descent with label noise drifts along the manifold as a gradient flow of the trace of the Hessian~\citep{blanc2020implicit,damian2021labelnoise,li2022zeroloss,ziyin2024symmetry}, the drift that in other settings carries it to sparse solutions~\citep{pesme2021diagonal} and simpler invariant subnetworks~\citep{kunin2021neuralmechanics,chen2023stochasticcollapse}; on the line of Proposition~\ref{prop:flat} we obtain that trace in closed form.

\begin{proposition}[Drift along the line]\label{prop:drift}
Let the trunk be fixed, absorb the biases into $z$ as a constant feature, let $\Sigma = \E[z z^\top]$, nonsingular, and $S = \operatorname{tr}\Sigma$, and let the loss be a squared error in an $m$-dimensional action.
On the zero-loss manifold the trace of the Gauss--Newton Hessian depends on the head parameters only through $m S (1 + \alpha^2) + \operatorname{tr}(W_h \Sigma W_h^\top)$.
In the small-step limit of stochastic gradient descent with label noise, whose dynamics on the manifold is the projected gradient flow of that trace (\citealp{li2022zeroloss}, Corollary 5.2), the gate and the head evolve, to leading order in $\alpha$, as
\begin{equation}\label{eq:rates}
  \frac{d\alpha}{dt} = -\,c\,\frac{m S - \operatorname{tr}(W_h \Sigma W_h^\top)}{1 + \|W_h\|_F^2}\,\alpha,
  \qquad
  \frac{d W_h}{dt} = -\,c\, W_h \Sigma ,
\end{equation}
with one positive constant $c$ set by the step size and the noise level; the gate decays exponentially whenever $\operatorname{tr}(W_h \Sigma W_h^\top) < m S$, which holds for every head with $\|W_h\|_2^2 < m$, and in every case the gate and the head both tend to zero.
\end{proposition}
\noindent The proof (Appendix~\ref{app:drift}) evaluates the trace on the line and projects its gradient onto the tangent space; the denominator is the squared length of the tangent vector along the line, and the head's output power $\operatorname{tr}(W_h \Sigma W_h^\top)$ enters the numerator because that vector is not orthogonal to the head directions.
The gate returns to zero in every case, so an expressible offset does not merely stop rising. The head decays per direction at the rate of one eigenvalue of $\Sigma$, whereas the gate decays at the rate of $m$ times their sum minus the head's output power, divided by $1 + \|W_h\|_F^2$, so for a small head the ratio of the two rates is of the order of the number of weights in the head, $m k$.
On the time scale on which the gate closes, the head is therefore nearly constant, and it decays only slowly afterwards.
This is the two-process picture the data show: the gate reads functional absorption, and the head's remaining norm is not evidence that absorption is incomplete.
Appendix~\ref{app:toy} checks the proposition in its own idealization, stochastic gradient descent with label noise started on the line; every prediction of the proposition holds there: the gate decays exponentially, within a factor of 1.6 of the flow (the excess being the finite-step correction) and in proportion to the noise variance; the head decays only after the gate has closed and more than a hundred times slower; and a linear head on a one-hidden-layer base, which that base can approximate but not reproduce, does not decay at all.
The agents, trained with Adam and a bootstrapped critic, show its pattern: the gate closes while most of the head's output remains (\S\ref{sec:arc}).

An inexpressible offset has no such line: if the base head cannot reproduce $\mu + \alpha h$ exactly for any choice of its parameters, then every minimizer of $\mathcal{L}$ that uses the offset needs $\alpha \neq 0$, and there is no flat direction through $\alpha = 0$ for the drift to follow.
The condition is exact reproduction, not approximation, as the linear head on a one-hidden-layer base shows in the minimal model and on the agents alike (\S\ref{sec:mechanism}).
A frozen base removes the line for a different reason: $W_\mu$ cannot move, so the split cannot change.
Detaching the head's gradient from the trunk changes neither the loss nor the line, and \S\ref{sec:mechanism} finds that it does not change whether $\alpha$ returns to zero.

\section{Experiments}
\label{sec:experiments}

\paragraph{Setup.}
In continuous control we train the shaped variants AS-TD3, AS-SAC, and AS-PPO of TD3, SAC, and PPO~\citep{fujimoto2018td3,haarnoja2018sac,schulman2017ppo} for 1M steps with 5 seeds on each of 20 environments, ten from MuJoCo v5 and ten from the DeepMind Control Suite~\citep{todorov2012mujoco,tassa2018dmcontrol}.
Scores are min-max normalized per environment with fixed anchors and aggregated by the interquartile mean (IQM) with 95\% stratified bootstrap intervals~\citep{agarwal2021rliable}; terminal scores average the last 20 evaluations (Appendix~\ref{app:eval}).
Mechanism experiments use six environments, the \emph{mechanism suite} (Hopper, HalfCheetah, Humanoid, finger/spin, walker/run, humanoid/walk), under TD3 and SAC with 5 seeds, 60 runs per arm (180 for the three-width MLP-head arm).

\subsection{What the head adds and what the base retains}
\label{sec:benefit}

\textbf{Training with the head raises the return of TD3.}
With the default hyperparameters and no per-task tuning, the IQM of the shaped TD3 actor exceeds that of TD3 by $+0.077$ [0.040, 0.116] over the 20 environments, and its mean is higher in 14 of them (the same implementation for both, 3 seeds per environment; Appendix~\ref{app:host}).
With an ensemble critic, the shaped SAC actor is above its host at every update-to-data ratio tested, by 0.029, 0.038 and 0.064 (Appendix~\ref{app:utd}), and on 26 Atari games the actor deployed without the head outscores the unshaped host (\S\ref{sec:discrete}).

\textbf{Removing the head costs almost nothing on a trainable base and much more on a frozen one.}
Over the 20 environments, the IQM of the per-run removal cost is $+0.006$ [0.002, 0.012] under TD3 and $+0.001$ [$-0.001$, 0.003] under SAC, and the per-environment interval over TD3 and SAC contains zero in 15 of 20 environments (Figure~\ref{fig:removal}); \S\ref{sec:removal} shows that the terminal residual predicts where the cost is paid.
With the trunk and base head frozen at initialization, the shaped actor scores 0.328 (TD3) and 0.371 (SAC), the base alone 0.011 and 0.009, and the IQM of the per-run removal cost is $+0.201$ and $+0.246$ (Figure~\ref{fig:headline}b).
Under PPO, whose critic is a state-value function, $\alpha$ stays within 0.05 of its initialization in 90 of the 100 runs, and removing the head costs $-0.0001$ [$-0.0014$, 0.0008]: the head is inert (\S\ref{sec:first-order}).

\begin{figure}[t]
  \centering
  \includegraphics[width=\textwidth]{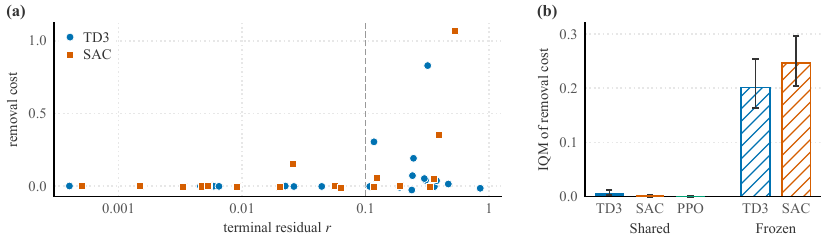}
  \caption{\textbf{The removal cost follows the residual, and is large on a frozen base.} (a) Removal cost in normalized return against the terminal residual $r$ of \S\ref{sec:measure}, one point per environment and algorithm (20 environments under TD3 and under SAC, mean over 5 seeds); the dashed line marks the threshold $r = 0.10$. (b) IQM of the per-run removal cost with its 95\% stratified bootstrap interval, for the shared head under TD3, SAC and PPO and for the frozen base under TD3 and SAC.}
  \label{fig:headline}
\end{figure}

\subsection{The arc: rise, then a decline that depends on the critic}
\label{sec:arc}

In the shared arm, in every environment of the mechanism suite, the amplitude $\|\delta_t\|$ traces an \emph{arc}: it rises and then falls (Figure~\ref{fig:arc}).
Under TD3 the median amplitude in four of the six environments peaks within 150K steps and falls to between 0.5\% and 25\% of its peak by 1M; finger/spin ends at 40\% and humanoid/walk at 35\%. Under SAC all six peak within 100K steps and end at or below 11\% of their peak.
The fall is due to $\alpha$: across the six environments the median decline of $|\alpha|$ from its peak is 98\%, while $\|h\|$ ends between 20\% and 94\% of its own peak, above half in ten of the twelve environment-algorithm cells, so the head retains most of its norm and its gate closes (Appendix~\ref{app:decomposition}).
Over all 20 environments the head engages in every one of the 200 TD3 and SAC runs, and the arc completes (\S\ref{sec:measure}) within 1M steps in 7 under TD3 and 13 under SAC; the removal cost that remains comes from the others (\S\ref{sec:removal}), and in each of six unfinished TD3 arcs run three times as long the residual falls further (Appendix~\ref{app:p4}).

\textbf{The critic can decide whether the gate closes.} On the four environments of the update-to-data experiment, the one TD3 gate that the twin critic leaves open at 1M, Hopper, closes under a ten-member ensemble critic~\citep{chen2021randomized}, its terminal-to-peak ratio of $|\alpha|$ falling from 21\% to 0.2\%, and the SAC gates close with either critic (Appendix~\ref{app:utd}).

\begin{figure}[t]
  \centering
  \includegraphics[width=\textwidth]{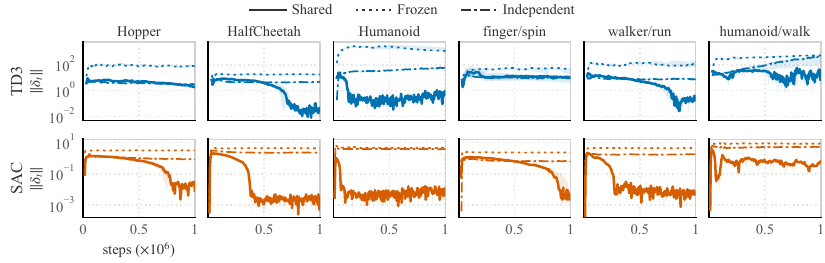}
  \caption{\textbf{The shared amplitude rises and falls, further under SAC, while the frozen and independent ones stay high.} $\|\delta_t\|$ per environment on the mechanism suite, median over 5 seeds with the interquartile range shaded, log scale shared within a row. On Humanoid, three of the five frozen-base SAC runs diverge by 220K steps, and that line is the median of the other two.}
  \label{fig:arc}
\end{figure}

\subsection{Removal cost along the amplitude trajectory}
\label{sec:removal}

\textbf{The cost of removing the head follows the residual, not the task,} as it should if absorption is what makes the head removable.
We split the 20 environments by the fixed criterion $r \le 0.10$: in the environments whose arc completed, removing the head costs $+0.0007$ [$-0.0007$, 0.0019] of normalized return under TD3 (7 environments) and $-0.0003$ [$-0.0014$, 0.0008] under SAC (13), nothing in either case (the IQM of the per-run cost, as in \S\ref{sec:benefit}). In the 13 and 7 environments whose arc did not complete, it costs $+0.028$ [0.013, 0.053] and $+0.030$ [0.005, 0.075].
Without any threshold, the size of an environment's removal cost rises with its residual: the rank correlation is 0.61 under TD3 ($p = 0.004$, bootstrap interval [0.22, 0.80] over environments) and 0.64 under SAC ($p = 0.003$, [0.29, 0.85]) (Figure~\ref{fig:headline}a).

\textbf{The same dependence appears inside a single run.} For every run of the mechanism suite we remove the head at four checkpoints set by a rule on the amplitude, and the cost is small during the rise, largest at or just after the peak, and indistinguishable from zero at the end, under both algorithms (Table~\ref{tab:checkpoint}). Around its peak the offset therefore raises the shaped actor's return by up to $+0.090$ under TD3 and $+0.109$ under SAC, the improvement that \S\ref{sec:first-order} describes.

\begin{table}[t]\setlength{\belowcaptionskip}{4pt}
  \caption{\textbf{Removal cost is largest at or just after the amplitude peak.} The head is removed at four points of each run, located by a rule that reads only the amplitude in action space (Appendix~\ref{app:eval}), on the mechanism suite; mean with 95\% bootstrap CI over 30 runs per algorithm.}
  \label{tab:checkpoint}
  \centering\small
  \begin{tabular}{lr@{\hspace{0.4em}}lr@{\hspace{0.4em}}l}
    \toprule
    Checkpoint & \multicolumn{2}{c}{TD3} & \multicolumn{2}{c}{SAC} \\
    \midrule
    rising & $+$0.015 & [0.001, 0.032] & $-$0.003 & [$-$0.009, 0.002] \\
    peak & $+$0.090 & [0.051, 0.134] & $+$0.045 & [0.009, 0.094] \\
    post-peak & $+$0.056 & [0.028, 0.091] & $+$0.109 & [0.041, 0.193] \\
    final & $+$0.017 & [$-$0.005, 0.049] & $-$0.012 & [$-$0.032, 0.002] \\
    \bottomrule
  \end{tabular}
\end{table}

\subsection{What absorption requires: seven arms}
\label{sec:mechanism}

Section~\ref{sec:flat-line} predicts that the offset is absorbed when the base can express it and is free to change, and that cutting the gradient between the heads does not prevent it; we measure the terminal residual and the removal cost of every arm of \S\ref{sec:arms}, and each arm behaves as predicted (Figure~\ref{fig:arms}).

Averaged over the twelve cells the shared arm's residual is 0.10 and removal costs $+0.003$.
Detaching the head from the trunk, so that no gradient from the offset reaches the representation, leaves the residual at 0.06 and the cost at $+0.0004$: the base absorbs the offset without receiving gradient from it, because it trains on the same critic at the same shaped action.
Giving the head its own trunk raises the residual to 0.55 (TD3) and 0.68 (SAC) and the cost to $+0.22$ and $+0.55$; the offset is still linear in a representation, but not in the representation the base reads, so the base cannot express it without first re-learning that representation.
Freezing the base leaves the residual at 0.68 under TD3 and 0.77 under SAC, never below 0.57 in any environment, with the removal cost of \S\ref{sec:benefit}.

\textbf{Linearity, not width, determines the outcome.}
The MLP-head arm changes only what the head can compute.
With a one-hidden-layer head of 16, 64, or 256 units, everything else unchanged, the residual stays between 0.58 and 0.68 under TD3 and between 0.74 and 0.83 under SAC, removal costs between $+0.44$ and $+0.78$, and none of the 36 (width, algorithm, environment) cells completes its arc, against eight of twelve for the shared arm.

\textbf{A head the base head can reproduce is absorbed, and one linear path determines whether it can.}
The last two arms move the nonlinearity to the base head and keep the shaping head linear.
On a one-hidden-layer base head of 64 units the head engages in all twelve (environment, algorithm) cells and its arc completes in none: the residual is 0.60 under TD3 and 0.82 under SAC, and removing the head costs $+0.38$ and $+0.35$.
Adding a zero-initialized linear path to that same base head, which restores the merge of \S\ref{sec:arms} and changes nothing else, lowers the residual in all twelve paired cells, its median from 0.79 to 0.012, and brings the cost to an IQM of $-0.0002$ [$-0.0007$, 0.0004] (Appendix~\ref{app:p1bc}).

\begin{figure}[t]
  \centering
  \includegraphics[width=\textwidth]{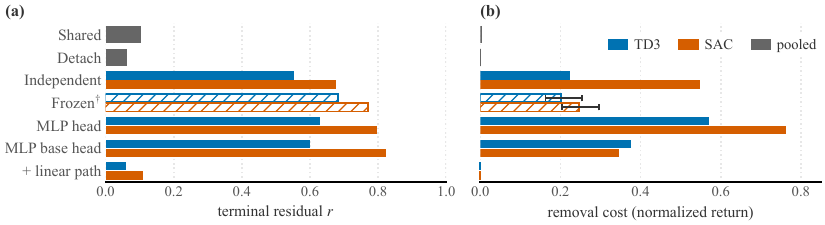}
  \caption{\textbf{Absorption requires expressibility and a trainable base.} Terminal residual $r$ (a) and removal cost (b) for each arm of \S\ref{sec:arms} on the mechanism suite, averaged over seeds and environments; the last row adds a zero-initialized linear path to the MLP base head. Shared and Detach are pooled over algorithms, and the MLP-head row is the mean over widths 16, 64, and 256, which agree within 0.11 in $r$. $\dagger$: for the frozen base, the residual is the mean over the runs that do not diverge (27 of 30 under SAC), and the cost is the IQM over the 20 environments of \S\ref{sec:benefit} with its 95\% stratified bootstrap interval.}
  \label{fig:arms}
\end{figure}

\subsection{Discrete actions}
\label{sec:discrete}

We ask the same question for discrete actions, a second geometry: on 26 Atari games with a discrete SAC host~\citep{zhou2022sdsac}, the offset acts on the logits and the base is read out by dropping it.
Under one protocol for all three policies, the base actor alone scores a human-normalized IQM of 0.183 [0.157, 0.216], the shaped actor 0.170, and the unshaped host 0.154; removing the head raises the IQM by 0.013 [0.003, 0.023] (Appendix~\ref{app:atari}).

\section{Related work}
\label{sec:related}

\paragraph{Reward shaping and the action channel.}
Potential-based shaping characterizes the reward modifications that leave the optimal policy unchanged and can be removed at deployment~\citep{ng1999reward}.
The action channel has no such invariance, since an offset changes the state distribution, so the question has to be answered by the training dynamics; our answer is a condition for absorption and the mechanism behind it (long form in Appendix~\ref{app:related}).

\paragraph{Folding a branch into a layer.}
Structural re-parameterization~\citep{ding2021repvgg,hu2022orepa,dlr2026} and linear over-parameterization~\citep{guo2020expandnets,arora2018implicit,du2018balanced} train a multi-branch block and fold it into one layer after training; our linear head admits the same merge (\S\ref{sec:arms}), which is our reference point for absorption. We show that training reaches it by itself, and only where an exact merge exists: a base head that can approximate the branch but not reproduce it keeps it, and one linear path that restores the merge restores absorption (\S\ref{sec:mechanism}). The condition thus predicts, for any co-trained branch, whether training will absorb it, and for a branch that admits no merge it predicts that the branch stays (Appendix~\ref{app:p9}); the gate reads either outcome while training runs.

\paragraph{Scalar-gated branches.}
In ReZero~\citep{bachlechner2021rezero} and LayerScale~\citep{touvron2021cait} a residual branch is scaled by a learnable gate that starts at or near zero, and the gate grows and stays open. Gates into a frozen backbone close through a dead-gradient regime~\citep{farazi2026gate}, attention sublayers learn to mute themselves~\citep{saikumar2025datafree}, and sparsity penalties prune explicitly~\citep{huang2018sss}. Our $\alpha$ closes with no penalty and only after it has opened, and the three states of \S\ref{sec:measure} distinguish these cases.

\paragraph{Training-only components.}
Privileged-sensor scaffolds train with a component and deploy without it~\citep{freed2024scaffolder}. Residual distillation~\citep{li2020residualdistillation}, tapered normalization~\citep{tapernorm2026}, decaying action priors~\citep{sood2024decap,sood2025apex}, and jump-start or attenuated residual policies~\citep{uchendu2023jsrl,trumpp2026arpo,bolychev2026agency} remove one on a schedule a designer sets; in every case the schedule, not the base, removes it.
Residual policy learning keeps the correction on a fixed base~\citep{silver2018residual,johannink2019residual,wang2025rpg,resfit2025,warprl2026}, the regime of our frozen arm; under the condition of \S\ref{sec:theory} no schedule is needed: an expressible branch on a trainable base drifts toward absorption without one.

\paragraph{Internalization and removability.}
\citet{tsilivis2026internalization} define internalization as absorbing an explicit procedure into weights of the same hypothesis class; their scaffold is removed by a curriculum, ours is absorbed, and we add the trainable-base condition, the flat-line mechanism, and an amplitude diagnostic.
Our checkpoint interventions find within a run that a parametric branch's removability depends on the training trajectory, as concurrent work finds for an annealable prior~\citep{yang2026removability}.

\section{Discussion}
\label{sec:discussion}

\paragraph{What removal means.}
Removal deploys the base actor as trained, with the head deleted; for the linear head the fold $W_\mu + \alpha W_h$ of \S\ref{sec:arms} is the closed-form reference point of the absorbed state: the offset $\delta$ is the distance of the trained base head from it, and absorption is the finding that training closes the distance on its own, with no merge performed.
Removal can also be free without absorption, for a head that never engaged (\S\ref{sec:measure}) or a redundant branch that collapses under strong noise (Appendix~\ref{app:toy}); the reading of \S\ref{sec:measure} covers those cases as well, and the condition speaks to absorption.
Where no fold exists, for a nonlinear or multiplicative head (Appendices~\ref{app:p2} and~\ref{app:p9}) or a base that is frozen or belongs to another module, that reading is the only one.

\paragraph{Inexpressible, not slow.}
At three million steps, three times the budget, the width-64 MLP head stays unabsorbed, median residual 0.73 and removal cost $+0.91$ (Appendix~\ref{app:p2}); for a linear offset the paired control of \S\ref{sec:mechanism} separates exact reproduction from budget.

\paragraph{Batch size in the agents.}
In a batch-size sweep on the agents, which train with Adam and a bootstrapped critic, a batch of 64 completes the arc before one of 1024 in five of six pairs, the direction Proposition~\ref{prop:drift} predicts (Appendix~\ref{app:p8}).

\paragraph{What the condition sorts.}
The condition reduces the question of when a co-trained branch can be removed for free to two questions about the agent: can the base head reproduce the branch exactly, and is the base trainable?
An expressible offset on a trainable base decays under gradient noise and needs no schedule (\S\ref{sec:mechanism}).
An offset on a fixed base stays, as in residual policy learning and our frozen arm.
An offset the base head cannot reproduce has no flat line to a closed gate and stays as well, as in every inexpressible arm of \S\ref{sec:mechanism} and in the same linear head composed multiplicatively, whose gate stays open in eleven of twelve suite cells (Appendix~\ref{app:p9}).
The recipe is short: add the head, watch its amplitude (in action space where the output map saturates, Appendix~\ref{app:p4}), and drop the head once the amplitude has fallen below a tenth of its peak.

\section{Conclusion}
\label{sec:conclusion}
Anyone who adds a branch to the action channel during training and removes it at deployment now has a condition that predicts whether training absorbs it.
We show that the condition is exact reproduction by a trainable base head, not capacity: a nonlinear head with more parameters is not absorbed, and in a paired control, one linear path added to a nonlinear base head restores absorption.
We identify the mechanism that explains why no schedule is needed: expressibility gives the loss a flat direction that leads to a closed gate, and gradient noise moves the parameters along it.
The gate reads either outcome while training runs: where the offset is absorbed it rises and then falls on its own, for deterministic and stochastic actors alike, and on 20 tasks removing the head costs almost nothing.
With default hyperparameters and no per-task tuning, the head also raises TD3's interquartile mean by 0.077 on the same 20 tasks, and on 26 Atari games the actor deployed without it outscores its host.
The condition can be asked of any co-trained branch on the action channel, whether it comes from a controller, a prior, or a residual policy, and, through the softmax case of Theorem~\ref{thm:first-order}, of discrete actions as well as continuous ones: Propositions~\ref{prop:flat} and~\ref{prop:drift} give the mechanism for a linear head on a linear base head, and the arms of \S\ref{sec:mechanism} and Appendix~\ref{app:p9} show that the branches absorbed are exactly those the base head can reproduce.
Policies absorb what they can express, and only that.

\bibliography{refs}
\bibliographystyle{iclr2027_conference}

\appendix

\AtBeginEnvironment{table}{\setlength{\belowcaptionskip}{4pt}}%
\section{Proofs}
\label{app:theorem1}

This appendix proves Theorem~\ref{thm:first-order} for a deterministic continuous actor and for the softmax over discrete actions, the continuous case written for the identity output map, with the Jacobian factor of a general smooth map such as the scaled $\tanh$ of TD3 supplied at its end, and then proves the corollary used in \S\ref{sec:first-order}. For a reparameterized Gaussian actor such as SAC's, $a = f(\mu(s) + \alpha h(s) + \sigma(s)\,\xi)$ with $\xi \sim \mathcal{N}(0, I)$ and $\sigma(s)$ bounded below by $\sigma_{\min} > 0$, the same four steps hold inside the expectation over $\xi$, with (A1) taken on the whole action set; the distribution shift is then bounded through the total-variation distance between the two Gaussians, at most $\alpha B_h / (2 \sigma_{\min})$, without (A2), and the first-order term is $\E_\xi\bigl[\nabla_a \Qpi(s, a)^\top J_f(\mu(s) + \sigma(s)\,\xi)\bigr]\, h(s)$, the derivative in~\eqref{eq:functional}.

\paragraph{Performance difference lemma.}
The derivation rests on the following lemma~\citep{kakade2002approximately}.

\begin{lemma}[Performance difference, \citealp{kakade2002approximately}]
\label{lem:pdl}
For any two policies $\pi$ and $\pi'$ acting in the same MDP,
\begin{equation}\label{eq:pdl-full}
  J(\pi') - J(\pi)
  \;=\; \frac{1}{1-\gamma}\;
  \E_{s \sim d^{\pi'}}\!\left[
    \E_{a \sim \pi'(\cdot|s)}\!\bigl[A^\pi(s,a)\bigr]
  \right],
\end{equation}
where $d^{\pi'}$ is the discounted state visitation distribution under $\pi'$ and $A^\pi(s,a) = Q^\pi(s,a) - V^\pi(s)$ is the advantage under $\pi$.
\end{lemma}

Equation~\ref{eq:pdl-full} is exact and holds without functional-form or smoothness assumptions on the policies.
The proof telescopes the discounted return; see \citet{kakade2002approximately}.

\paragraph{Regularity conditions.}

\begin{assumption}[Continuous action space]
\label{asm:continuous}
\leavevmode
\begin{enumerate}
\item[\textup{(A1)}]
  $\Qpi(s,a)$ is twice continuously differentiable in $a$ in a neighborhood of $a = \mu(s)$ of a common radius $r_0 > 0$ for all $s$, with $\|\nabla^2_a \Qpi\| \leq M_Q$ and $\|\nabla_a \Qpi(s, \mu(s))\| \leq G_Q$ for all $s$.

\item[\textup{(A2)}]
  The transition kernel satisfies $\mathrm{TV}\bigl(P(\cdot|s,a_1),\, P(\cdot|s,a_2)\bigr) \leq L_P \|a_1 - a_2\|$ for all $s, a_1, a_2$.

\item[\textup{(A3)}]
  The shaping head output is bounded, $\|h(s)\| \leq B_h$ for all $s$.

\item[\textup{(A4)}]
  Rewards are bounded, $|r(s,a)| \leq R_{\max}$.

\item[\textup{(A5)}]
  The action gradient $\nabla_a \Qpi(s, \mu(s))$ is continuous in the parameters of the base policy, for every $s$.
\end{enumerate}
\end{assumption}

\begin{assumption}[Discrete action space]
\label{asm:discrete}
The action space is finite, $|\cA| < \infty$, and (A3) and (A4) of Assumption~\ref{asm:continuous} hold.
The state-distribution shift is then bounded directly through the total-variation distance between policies, which is $O(\alpha)$ by the Lipschitz property of the softmax, without requiring smoothness of the transition kernel in the action argument.
For the limit in Appendix~\ref{app:corollary}, $\Qpi(s, a)$ is continuous in the parameters of the base policy, for every $s$ and $a$.
\end{assumption}

\paragraph{Derivation.}

\begin{proof}[Proof of Theorem~\ref{thm:first-order}]

\medskip\noindent\textbf{Continuous deterministic policies.}

\emph{Step 1: apply the performance difference lemma.}\;
With $\pi' = \pi_\alpha$, the shaped policy at gate value $\alpha$, and $\pi = \pimu$,
\begin{equation}\label{eq:pdl-applied-cont}
  G(\alpha) - G(0)
  = \frac{1}{1-\gamma}\;
  \E_{s \sim d^{\pi_\alpha}}\!\bigl[
    A^{\pimu}\!\bigl(s,\, \mu(s) + \alpha\, h(s)\bigr)
  \bigr].
\end{equation}

\emph{Step 2: separate the state-distribution shift.}\;
The difference $d^{\pi_\alpha} - \dpimu$ introduces a second-order correction:
\begin{equation}\label{eq:dist-shift}
  \E_{d^{\pi_\alpha}}[\psi(s)]
  = \E_{\dpimu}[\psi(s)]
  + \underbrace{
    \E_{d^{\pi_\alpha} - \dpimu}[\psi(s)]
  }_{O(\alpha)}.
\end{equation}
Under Assumption~\ref{asm:continuous}(A2) the total-variation distance satisfies $\mathrm{TV}(d^{\pi_\alpha}, \dpimu) = O\!\bigl(\gamma L_P \alpha B_h / (1-\gamma)\bigr)$, obtained by recursively bounding the per-step distribution shift with the Lipschitz condition on $P$.
The integrand $A^{\pimu}(s, \mu(s) + \alpha h(s))$ vanishes at $\alpha = 0$ and is $O(\alpha)$ by the Taylor expansion of Step~3, so the cross-term is $O(\alpha) \times O(\alpha) = O(\alpha^2)$:
\begin{equation}\label{eq:dist-shift-bound}
  \E_{d^{\pi_\alpha} - \dpimu}\!\bigl[
    A^{\pimu}\!\bigl(s,\, \mu(s) + \alpha h(s)\bigr)
  \bigr]
  = O\!\left(\frac{\gamma L_P \alpha^2 B_h^2
    G_Q}{1-\gamma}\right).
\end{equation}

\emph{Step 3: Taylor-expand the advantage.}\;
The advantage at the base action is zero, $A^{\pimu}(s, \mu(s)) = \Qpi(s, \mu(s)) - V^{\pimu}(s) = 0$.
Under Assumption~\ref{asm:continuous}(A1) the first-order expansion around $a = \mu(s)$ gives
\begin{align}\label{eq:taylor-advantage}
  A^{\pimu}\!\bigl(s,\, \mu(s) + \alpha h(s)\bigr)
  &= \nabla_a \Qpi\!(s, \mu(s))^\top \cdot \alpha \, h(s)
  + O(\alpha^2 M_Q B_h^2),
\end{align}
using $\nabla_a A^{\pimu} = \nabla_a \Qpi$, since $V^{\pimu}(s)$ does not depend on $a$.

\emph{Step 4: combine.}\;
Substituting~\eqref{eq:dist-shift-bound} and~\eqref{eq:taylor-advantage} into~\eqref{eq:pdl-applied-cont},
\begin{align}
  G(\alpha) - G(0)
  &= \frac{1}{1-\gamma}\;
  \E_{\dpimu}\!\bigl[
    \alpha\, \nabla_a \Qpi\!(s,\mu(s))^\top h(s)
  \bigr]
  + O(\alpha^2) \notag \\
  &= \frac{\alpha}{1-\gamma}\;
  \E_{\dpimu}\!\bigl[\Phi_h(s)\bigr]
  + R(\alpha),
  \label{eq:thm1-cont-final}
\end{align}
where $\Phi_h(s) = \nabla_a \Qpi\!(s, \mu(s))^\top h(s)$ is the continuous case of~\eqref{eq:functional}, and the remainder satisfies $R(\alpha) = O(\alpha^2)$ with leading constant $C_{\text{cont}} = \frac{1}{1-\gamma} \bigl(\tfrac{1}{2}M_Q B_h^2 + \tfrac{2\gamma L_P G_Q B_h^2}{1-\gamma}\bigr)$, the first term being the Taylor remainder of Step~3 and the second the state-distribution shift of Step~2.

\emph{Output maps other than the identity.}\;
TD3 acts with $a = f(\mu(s) + \alpha h(s))$ for $f$ a $\tanh$ scaled to the action bounds (Equation~\ref{eq:shaped}), and the four steps go through for any output map that is $C^2$ with $\|J_f\| \le L_f$ and $\|\nabla^2 f\| \le M_f$, since $f(\mu + \alpha h) - f(\mu) = \alpha\, J_f(\mu)\, h + O(\alpha^2 M_f B_h^2)$.
Step~2 holds with $B_h$ replaced by $L_f B_h$; Step~3 expands $\Qpi$ around $a = f(\mu(s))$, where (A1) is now taken, with $h$ replaced by $J_f(\mu(s))\, h$ and the remainder constant $\tfrac{1}{2} M_Q B_h^2$ by $\tfrac{1}{2}\bigl(M_Q L_f^2 + M_f G_Q\bigr) B_h^2$; and the first-order term becomes $\Phi_h(s) = \nabla_a \Qpi\bigl(s, f(\mu(s))\bigr)^\top J_f(\mu(s))\, h(s)$, which is the derivative in~\eqref{eq:functional} by the chain rule.
The identity map is the case $J_f = I$, $M_f = 0$; for the scaled $\tanh$, $L_f$ is the action scale and $M_f$ is $4/(3\sqrt{3}) \approx 0.77$ times it.

\medskip\noindent\textbf{Softmax policies over discrete actions.}

\emph{Step 1: apply the performance difference lemma.}\;
With $\pi_\alpha(a|s) = \mathrm{softmax}\bigl(\mu(s) + \alpha \cdot h(s)\bigr)_a$, where $\mu(s)$ is the vector of base logits,
\begin{equation}\label{eq:pdl-applied-disc}
  G(\alpha) - G(0)
  = \frac{1}{1-\gamma}\;
  \E_{s \sim d^{\pi_\alpha}}\!\left[
    \sum_a \bigl(\pi_\alpha(a|s) - \pimu(a|s)\bigr)\,
    \Qpi(s,a)
  \right].
\end{equation}

\emph{Step 2: separate the state-distribution shift.}\;
Under Assumption~\ref{asm:discrete}, $\mathrm{TV}(\pi_\alpha(\cdot|s), \pimu(\cdot|s)) = O(\alpha)$ by the Lipschitz property of the softmax, so $\mathrm{TV}(d^{\pi_\alpha}, \dpimu) = O(\gamma\alpha/(1-\gamma))$.
The cross-term is again $O(\alpha^2)$.

\emph{Step 3: expand the policy difference.}\;
The softmax Jacobian $\partial \pi(a|s)/\partial \mu_b(s) = \pi(a|s)\bigl(\mathbf{1}\{a = b\} - \pi(b|s)\bigr)$ gives
\begin{equation}\label{eq:softmax-expansion}
  \pi_\alpha(a|s) - \pimu(a|s)
  = \alpha \cdot \pimu(a|s)\,
  \bigl(h_a(s) - \E_{\pimu}[h(s)]\bigr)
  + O(\alpha^2),
\end{equation}
and substituting into the inner sum,
\begin{align}
  \sum_a \bigl(\pi_\alpha(a|s) - \pimu(a|s)\bigr)\, \Qpi(s,a)
  &= \alpha \sum_a \pimu(a|s)\,
  \bigl(h_a - \E_{\pimu}[h]\bigr)\, \Qpi(s,a)
  + O(\alpha^2) \notag \\
  &= \alpha\, \operatorname{Cov}_{\pimu(\cdot|s)}\!\bigl(h(s),\, \Qpi(s,\cdot)\bigr)
  + O(\alpha^2).
  \label{eq:cov-expansion}
\end{align}

\emph{Step 4: combine.}\;
\begin{equation}\label{eq:thm1-disc-final}
  G(\alpha) - G(0)
  = \frac{\alpha}{1-\gamma}\;
  \E_{\dpimu}\!\bigl[\Phi_h(s)\bigr]
  + R(\alpha),
\end{equation}
where $\Phi_h(s) = \operatorname{Cov}_{\pimu(\cdot|s)}\!\bigl(h(s),\, \Qpi(s,\cdot)\bigr)$ is the softmax case of~\eqref{eq:functional} and $|R(\alpha)| \leq C_{\text{disc}} \cdot \alpha^2$ with $C_{\text{disc}} = B_h^2 R_{\max} / (1-\gamma)^3$.
\end{proof}

\subsection{Self-diminishing corollary}
\label{app:corollary}

\begin{corollary}[Self-diminishing property]
\label{cor:self-diminishing}
If the base policy $\pimu$ is optimal, then $\Phi_h(s) = 0$ for all shaping heads $h$ and all states $s$, for deterministic and softmax policies; for a reparameterized Gaussian actor the gradient of the full actor loss with respect to $h$ vanishes at the soft optimum.
\end{corollary}

\begin{proof}
\emph{Continuous deterministic.}\;
At optimality the first-order necessary condition on the pre-activation gives
\[
  J_f(\mu(s))^\top \nabla_a \Qpi(s, f(\mu(s))) = 0,
\]
so $\Phi_h(s) = \nabla_a \Qpi(s, f(\mu(s)))^\top J_f(\mu(s))\, h(s) = 0$; with the identity map this reads $\nabla_a \Qpi(s, \mu(s)) = 0$ and $\Phi_h(s) = \nabla_a \Qpi(s, \mu(s))^\top h(s) = 0$.

\emph{Discrete.}\;
A softmax policy has full support, so it is optimal at $s$ only if every action attains $\max_a \Qpi(s,a)$, and then $\Qpi(s,\cdot)$ is constant and the covariance vanishes. As the base policy approaches optimality its mass concentrates on $\cA^*_s = \argmax_a \Qpi(s,a)$, and writing the covariance with $\max_b \Qpi(s,b)$ subtracted from $\Qpi(s,\cdot)$ gives $|\Phi_h(s)| \le 4 B_h R_{\max}\, \pimu(\cA \setminus \cA^*_s \mid s)/(1-\gamma) \to 0$, under (A3) and (A4).

\emph{SAC (entropy-regularized).}\;
At the soft optimum the stationarity condition for the maximum-entropy objective gives $\nabla_\mu\bigl(\E[\Qpi] + \beta\, \mathcal{H}(\pi)\bigr) = 0$.
The $Q$-part of this gradient, projected onto $h$, is precisely $\Phi_h^{\text{SAC}}(s)$, and the entropy part is $\beta\, (\nabla_\mu \mathcal{H})^\top h$.
At soft optimality these two terms cancel, so the total gradient signal reaching $h$ from the SAC actor loss vanishes.
While $\Phi_h^{\text{SAC}}(s)$ alone need not be exactly zero, since it equals $-\beta(\nabla_\mu \mathcal{H})^\top h$, the training dynamics drive $h$ through the full actor loss, and the gradient of that full loss with respect to $h$ vanishes at the soft optimum.
\end{proof}

The corollary extends from an optimal base to one that approaches optimality under Assumption~\ref{asm:continuous}(A5), taken at $a = f(\mu(s))$ for an output map other than the identity, and its discrete counterpart in Assumption~\ref{asm:discrete}.
For a fixed head $h$ and state $s$, $\Phi_h(s)$ is then continuous in the parameters of the base policy, since $J_f$ is continuous, $\mu(s)$ is continuous in those parameters, and the softmax covariance is continuous in the logits and in $\Qpi(s, \cdot)$; hence $\Phi_h(s) \to 0$ as the parameters converge to those of an optimal base policy.
Under SAC the statement applies to the gradient of the full actor loss with respect to $h$, as in the SAC case of the proof, and not to $\Phi_h^{\text{SAC}}$ alone.

The corollary shows that the per-state functional is not an independent signal but a measure of the base policy's sub-optimality.
As the base policy improves through training, $\Phi_h$ diminishes, weakening the gradient signal to both $h$ and $\alpha$.

\subsection{Flat line}
\label{app:flat}

\begin{proof}[Proof of Proposition~\ref{prop:flat}]
The merged readout $(\bar W, \bar b)$ is unchanged along the line, and the trunk $\phi$ is unchanged, so the shaped action, hence the loss, is unchanged.
Differentiating $\mathcal{L}$ twice with respect to $W_h$ through $\bar W = W_\mu + \alpha W_h$ gives $\nabla^2_{W_h} \mathcal{L} = \alpha^2\, \nabla^2_{\bar W} \mathcal{L}$, whose trace is $\alpha^2$ times a quantity that depends on $\theta$ only through $(\phi, \bar W, \bar b)$, all constant along the line; at a minimizer of $\mathcal{L}$ the Hessian with respect to $(\bar W, \bar b)$ is positive semidefinite, because any change of $\bar W$ can be made through $W_\mu$ alone, so that quantity is nonnegative and the trace is least at $\alpha + \lambda = 0$.
\end{proof}

\subsection{Drift along the line}
\label{app:drift}

\paragraph{Setting.}
The trunk is fixed and its features are $z(s) \in \R^k$ with the constant feature appended, so that biases are columns of the weight matrices; $\Sigma = \E[z z^\top]$, assumed nonsingular, and $S = \operatorname{tr}\Sigma$, expectations over a fixed set of training states.
The parameters are $\theta = (W_\mu, W_h, \alpha)$ with $W_\mu, W_h \in \R^{m \times k}$, the model is $f_\theta(s) = (W_\mu + \alpha W_h)\, z(s)$, and the loss is $\mathcal{L}(\theta) = \tfrac{1}{2}\E\|f_\theta(s) - y(s)\|^2$ with the target perturbed by independent noise at every step of stochastic gradient descent.
Let $\Gamma$ be the zero-loss manifold, assumed non-empty, and let $\bar W = W_\mu + \alpha W_h$; on $\Gamma$ the Hessian of $\mathcal{L}$ equals the Gauss--Newton matrix $\E[J^\top J]$ with $J = \partial f_\theta / \partial \theta$.
\citet{li2022zeroloss}, Corollary 5.2, give the limiting dynamics of label-noise stochastic gradient descent on $\Gamma$, as the step size $\eta \to 0$ over $\Theta(\eta^{-2})$ steps, as the projected gradient flow $\dot\theta = -\,c\, P_\theta \nabla R(\theta)$, where $R(\theta) = \operatorname{tr}\E[J^\top J]$, $P_\theta$ is the projection onto the tangent space of $\Gamma$, and $c > 0$ collects the step size and the noise variance.
The tangent space is the null space of the Hessian, and the normal space its range, so $P_\theta$ is the orthogonal projection.

\paragraph{The trace.}
The blocks of $J$ are $\partial f/\partial W_\mu = z^\top \otimes I_m$, $\partial f/\partial W_h = \alpha\, z^\top \otimes I_m$, and $\partial f/\partial \alpha = W_h z$.
Hence
\begin{equation}\label{eq:trace-on-line}
  R(\theta) = m S + \alpha^2 m S + \E\|W_h z\|^2 = m S\,(1 + \alpha^2) + \operatorname{tr}(W_h \Sigma W_h^\top),
\end{equation}
which does not depend on $W_\mu$; with a trainable trunk the additional terms depend on the head only through $\bar W$ and are constant along the line.

\paragraph{The tangent space and the projection.}
A displacement $(dW_\mu, dW_h, d\alpha)$ is tangent to $\Gamma$ when it leaves $\bar W$ unchanged, $dW_\mu + \alpha\, dW_h + d\alpha\, W_h = 0$.
Two families span it: the line of Proposition~\ref{prop:flat}, $u = (-W_h, 0, 1)$, of squared length $1 + \|W_h\|_F^2$, and the head directions $v_V = (-\alpha V, V, 0)$ for $V \in \R^{m \times k}$.
The gradient of $R$ has components $\partial R/\partial W_\mu = 0$, $\partial R/\partial W_h = 2 W_h \Sigma$, and $\partial R/\partial \alpha = 2 m S \alpha$.
The two families are not orthogonal, $\langle u, v_V \rangle = \alpha \langle W_h, V \rangle$, and because the head gradient is of order one this coupling contributes to the gate's equation at the same order as the gate gradient itself. Orthogonalizing $u$ against the head directions gives $u_\perp = \bigl(-W_h/(1+\alpha^2),\, -\alpha W_h/(1+\alpha^2),\, 1\bigr)$ with $\|u_\perp\|^2 = (1 + \alpha^2 + \|W_h\|_F^2)/(1+\alpha^2)$, and the $\alpha$ component of the exact projection is $\langle \nabla R, u_\perp \rangle / \|u_\perp\|^2 = 2\alpha\,\bigl(m S (1+\alpha^2) - \operatorname{tr}(W_h \Sigma W_h^\top)\bigr) / (1 + \alpha^2 + \|W_h\|_F^2)$; solving the normal-space equation $(1+\alpha^2) M + \langle W_h, M\rangle W_h = 2\alpha (W_h \Sigma + m S\, W_h)$ for the normal component gives the same value, and the head component $\bigl(2 W_h \Sigma - \alpha\,(P_\theta \nabla R)_\alpha W_h\bigr)/(1+\alpha^2)$, which is the full head gradient $2 W_h \Sigma$ to leading order in $\alpha$.
Reading off the $\alpha$ and $W_h$ components of $-\,c\,P_\theta \nabla R$ gives, with the factor $2$ absorbed into $c$, the exact gate equation and the head equation to leading order in $\alpha$,
\begin{equation}\label{eq:rates-appendix}
  \dot\alpha = -\,c\,\frac{m S (1 + \alpha^2) - \operatorname{tr}(W_h \Sigma W_h^\top)}{1 + \alpha^2 + \|W_h\|_F^2}\,\alpha, \qquad \dot W_h = -\,c\, W_h \Sigma, \qquad \dot W_\mu = -\,\dot\alpha\, W_h - \alpha\, \dot W_h,
\end{equation}
the last equation being the statement that the flow stays on $\Gamma$.
Dropping the $\alpha^2$ terms of the first equation gives Proposition~\ref{prop:drift}.

\paragraph{Remarks.}
The gate decays at the rate of $m S - \operatorname{tr}(W_h \Sigma W_h^\top)$, the summed curvature of all $m k$ head weights less the head's output power, divided by the squared length of the line's tangent vector; the rate is positive whenever $\|W_h\|_2^2 < m$, since $\operatorname{tr}(W_h \Sigma W_h^\top) \le \|W_h\|_2^2 S$, and when the output power exceeds $m S$ the gate first rises and then decays: the trace restricted to $\Gamma$ is a strictly convex function of $(\alpha, W_h)$ whose only critical point is $(0, 0)$, and the projected flow is the gradient flow of that coercive trace on $\Gamma$, so from every starting point it converges to $\alpha = 0$ and $W_h = 0$. The head decays per eigen-direction of $\Sigma$ at the rate of one eigenvalue.
Since $S$ is the sum of the $k$ eigenvalues, the ratio of the two rates is $\bigl(m k - \operatorname{tr}(W_h \Sigma W_h^\top)/\bar\lambda\bigr)$ times the ratio of the mean eigenvalue $\bar\lambda$ to the eigenvalue of the direction in question, divided by $1 + \|W_h\|_F^2$; for a small head it is of the order of the number of weights in the head.
The mechanism behind the head's slow decay is visible in the projection: the noise on $\alpha$ has variance proportional to $\E\|W_h z\|^2$ and does not vanish at $\alpha = 0$; a kick of size $\zeta$ in $\alpha$ moves $\bar W$ by $\zeta W_h$, and the restoring gradient on $W_h$ is then $\zeta\,(\zeta W_h \Sigma)$, whose sign does not depend on the sign of $\zeta$.
The coupling $\alpha \langle W_h, V \rangle$ between the two tangent families enters the gate's equation at first order in $\alpha$, which is the output-power term above, and the head's equation at order $\alpha^2$; while the gate is of order one the latter is comparable to the head term when $\alpha^2 (m S - \operatorname{tr}(W_h \Sigma W_h^\top))/(1 + \|W_h\|_F^2)$ is comparable to the eigenvalues of $\Sigma$, it enters with the opposite sign, and it dominates until the gate has closed.
The head's norm is therefore stationary or slightly growing while the gate decays, and shrinks only afterwards; the minimal model measures $+2.1 \times 10^{-9}$ per step for $\ln\|W_h\|_F$ at $|\alpha| \approx 0.2$ against $+1.9 \times 10^{-9}$ predicted with the coupling kept (Appendix~\ref{app:toy}).
An inexpressible head that the targets need has no zero-loss point with $\alpha = 0$, so the flow has no such point to reach; a frozen base removes $W_\mu$ from the coordinates, so the zero-loss set is $\{\alpha W_h = D\}$ for the fixed offset $D$ that the base needs, which contains no point with $\alpha = 0$; along it the trace over the trained parameters, $\alpha^2 m S + \operatorname{tr}(D \Sigma D^\top)/\alpha^2$, is minimized at $\alpha^4 = \operatorname{tr}(D \Sigma D^\top)/(m S)$, where the gate stays.

\section{Minimal model of absorption}
\label{app:toy}

\paragraph{Model.}
States are $s \sim \mathcal{N}(0, I_{16})$.
A trainable trunk $\phi$ (two hidden layers of 64 units, ReLU) computes $z = \phi(s)$; the base head is linear, $\mu = W_\mu z + b_\mu \in \R^4$; the shaping head is a zero-initialized linear map of $z$; the gate is initialized as $\alpha \sim U(-0.01, 0.01)$; the action is $a = \mu + \alpha h$ with no squashing, so that the merge identity of \S\ref{sec:arms} is exact.
The critic is exact and fixed, $Q(s, a) = -\tfrac{1}{2}\|a - g(s)\|^2$, where $g$ is a fixed random two-layer tanh network with 32 hidden units and outputs scaled to unit root mean square, and the loss is $\mathcal{L} = \tfrac{1}{2}\E\|a - g(s)\|^2$, minimized over every trainable parameter by Adam with learning rate $3 \times 10^{-4}$, the optimizer of the agents.
The seven configurations of Table~\ref{tab:toy} are the first five arms of \S\ref{sec:arms} plus two MLP-base controls, an MLP head and a linear head on a one-hidden-layer base head of 32 units; the table's last row repeats the shared arm under plain SGD.
Gradient noise is the controlled variable: N0 uses one fixed batch of 4096 states, so the gradient is exact and training is deterministic; N1 draws a fresh batch of 64 states at every step; N2 adds Gaussian noise of standard deviation 0.3 to the target $g(s)$ at every step, a stand-in for a noisy critic.
Every configuration runs for 100K steps with 3 seeds; $\|\delta\|$, $\|h\|$, and the removal cost $\mathcal{L}(\mu) - \mathcal{L}(\mu + \alpha h)$ are logged every 500 steps on a fixed evaluation set of 2048 states; the residual $r$ uses the last 5\% of steps over the peak, with completion at $r \le 0.10$ as in \S\ref{sec:measure}.

\begin{table}[htbp]
  \caption{%
    \textbf{Minimal model.} Terminal residual $r$ (median over 3 seeds), number of seeds completing ($r \le 0.10$), and the fraction of its peak norm that the head retains at the end ($\|h\|_{\text{end}} / \|h\|_{\text{peak}}$, median), for each configuration under three noise levels; Adam in every row but the last, which uses plain SGD (learning rate $10^{-2}$).
  }
  \label{tab:toy}
  \centering
  \small
  \setlength{\tabcolsep}{5pt}\begin{tabular}{@{}lcccccclcc@{}}
\toprule
& \multicolumn{3}{c}{N0: full batch} & \multicolumn{3}{c}{N1: minibatch} & \multicolumn{3}{c}{N2: minibatch + target noise} \\
\cmidrule(lr){2-4}\cmidrule(lr){5-7}\cmidrule(lr){8-10}
Arm & $r$ & done & $h$ kept & $r$ & done & $h$ kept & \multicolumn{1}{c}{$r$} & done & $h$ kept \\
\midrule
Shared, linear head & 0.997 & 0/3 & 0.84 & 0.695 & 0/3 & 0.96 & 0.0014 & 3/3 & 0.14 \\
Detach & 0.996 & 0/3 & 0.85 & 0.687 & 0/3 & 0.98 & 0.0012 & 3/3 & 0.13 \\
Independent trunk & 0.999 & 0/3 & 0.66 & 0.996 & 0/3 & 0.80 & 0.979 & 0/3 & 0.99 \\
Frozen base & 1.000 & 0/3 & 0.64 & 0.993 & 0/3 & 0.51 & 0.993 & 0/3 & 0.51 \\
MLP head, linear base & 0.995 & 0/3 & 0.60 & 0.777 & 0/3 & 0.87 & 0.086 & 2/3 & 0.13 \\
Linear head, MLP base & 0.999 & 0/3 & 0.91 & 0.875 & 0/3 & 0.99 & 0.983 & 0/3 & 0.99 \\
MLP head, MLP base & 0.996 & 0/3 & 0.84 & 0.866 & 0/3 & 0.70 & 0.120 & 1/3 & 0.22 \\
\addlinespace
Shared, linear head (SGD) & \multicolumn{9}{l}{non-engaged at every noise level: peak amplitude below $6 \times 10^{-4}$} \\
\bottomrule
\end{tabular}

\end{table}

\paragraph{What the model shows.}
Without gradient noise nothing returns (Figure~\ref{fig:toy}): under N0 every Adam configuration ends with $r \ge 0.995$, the shared head included, whose $\alpha$ rises throughout the run, to $0.43$ at 100K.
With minibatch noise the shared head turns over: $\alpha$ peaks near $0.28$ in the first half of the run and declines from there, to $0.20$ at 100K, with $r = 0.70$ still falling at the budget and the head retaining 96\% of its norm.
With target noise added it returns within the budget, $r = 0.0014$ and completion at about 58K steps in every seed, and there the head's norm falls as well, to 14\% of its peak.
The drift therefore acts on $\alpha$ first, as the rate ratio of \S\ref{sec:flat-line} predicts, and reaches $W_h$ once the noise is strong; the trained agents, whose $|\alpha|$ falls by 98\% while $\|h\|$ keeps 20\% to 94\% (\S\ref{sec:arc}), sit between the two noisy regimes.
Detaching the head from the trunk changes nothing at any noise level (differences in $r$ below 0.01).
The independent-trunk and frozen-base heads never return ($r \ge 0.97$ and $\ge 0.99$), and removing the head from the frozen base costs 1.65 in loss against at most 0.03 for the shared head.

\paragraph{What decides absorption is exact expressibility.}
The linear head on the one-hidden-layer base is the decisive control: the base can approximate a linear function of $z$ but cannot reproduce $\mu + \alpha h$ exactly, and its head keeps its amplitude at every noise level ($r$ of 0.88 to 1.00, head norm retained at 91\% to 99\%), even under the target noise that closes the linear-on-linear gate in every seed.
The MLP head on the MLP base behaves the same way at N1 ($r = 0.87$) and only partly returns at N2 ($r = 0.12$, one seed of three); a sum of two one-hidden-layer networks of the same width is not a network of that width, so no exact line exists there either.
The MLP head on the linear base is not absorbed under minibatch noise ($r = 0.78$) but does return under target noise ($r = 0.086$, two seeds of three) with its norm reduced to 13\% of peak and a removal cost of $0.0004$: the branch collapses while the base learns the task on its own, the stochastic collapse of a redundant branch~\citep{chen2023stochasticcollapse} rather than absorption along a flat line.
The diagnostic of \S\ref{sec:removal} holds in that case too: the residual is small and the removal free, though the mechanism differs; the agents' MLP heads do not collapse at three times the budget (Appendix~\ref{app:p2}).
The two MLP-base controls thus isolate exact expressibility as the condition, which the agent test of Appendix~\ref{app:p1bc} confirms.

\paragraph{The optimizer.}
Under plain SGD at learning rate $10^{-2}$ the head never engages: the peak amplitude stays below $6 \times 10^{-4}$ against 0.15 to 0.62 under Adam, while the base alone drives the loss from 2.0 to 0.01.
The zero-initialized head sits at a point where the gradient with respect to $W_h$ carries a factor $\alpha$ and the gradient with respect to $\alpha$ a factor $\|h\|$, both near zero; Adam's per-parameter normalization turns those small but consistent gradients into steps of the size of the learning rate, and plain SGD does not.

\paragraph{Drift check.}
To test Proposition~\ref{prop:drift} directly we start on the line.
A linear-head actor is trained by Adam without noise on 512 fixed states until the training loss is below $10^{-6}$, which is reachable because the model has more parameters than constraints, leaving $\alpha \approx -0.21$ and $\|W_h\|_F \approx 3.4$; from there training continues with plain stochastic gradient descent on the full batch, label noise of standard deviation $\sigma$ added to the target at every step, the trunk frozen so that $\Sigma$ is fixed, three seeds.
At step size $10^{-2}$ over 200K steps the measured rate of $\ln|\alpha|$ is $-8.2 \times 10^{-8}$ per step at $\sigma = 0.3$ and $-9.9 \times 10^{-9}$ at $\sigma = 0.1$, against $-7.5 \times 10^{-8}$ and $-8.3 \times 10^{-9}$ from the projected flow of Appendix~\ref{app:drift}, output-power term included, evaluated along the trajectory, ratios of 1.09 and 1.19, and the two rates differ by a factor of 8.25 for a ninefold change in $\sigma^2$.
Over this budget the gate moves by at most 2\%, which fixes its initial rate; the head's norm does not shrink but grows slightly, at $+2.1 \times 10^{-9}$ per step against $+1.9 \times 10^{-9}$ predicted with the coupling of Appendix~\ref{app:drift} kept.
With $\sigma = 0$ nothing moves, and the MLP head, the frozen base, and the linear head on an MLP base keep at least 96\% of $|\alpha|$ under the same noise, the last of them growing by 25\%.

A run from the same starting points raises the step size to $0.05$ and the budget to 3M steps so that the drift outruns the gate's own fluctuations.
There the gate decays exponentially: $\ln|\alpha|$ is linear in the step with $R^2 = 0.996$ at $\sigma = 0.3$ and $0.990$ at $\sigma = 0.45$, reaching $|\alpha| < 0.02$ after 0.80M and 0.36M steps, at rates of $-2.8 \times 10^{-6}$ and $-6.2 \times 10^{-6}$ per step against $-1.8 \times 10^{-6}$ and $-4.1 \times 10^{-6}$ from the flow (ratios 1.59 and 1.57; at step size $0.03$ the ratio falls to 1.26, so the excess is the finite-step correction: a discrete-time correction to the flow, which enlarges the stationary variance of each Hessian mode by $1/(1 - \eta\lambda_i/2)$, reproduces the measured ratio of each of the nine runs to within 1\%), and the two rates differ by a factor of 2.24 for a $\sigma^2$ ratio of 2.25.
Once the gate has closed the head begins to shrink, at $-1.8 \times 10^{-8}$ and $-2.7 \times 10^{-8}$ per step ($R^2$ of 0.99 and 0.97), 1.03 and 0.97 times the rate of the flow; the gate therefore decays 150 and 228 times faster than the head against predicted ratios of 101 and 146, and the head still holds 96\% and 92\% of its norm at 3M steps.
Under the same step size and noise the MLP head keeps 65\% to 70\% of its gate with no exponential form ($R^2 \le 0.92$), and the linear head on an MLP base grows its gate by a factor of 3.5 and stays there.

\begin{figure}[htbp]
  \centering
  \includegraphics[width=\textwidth]{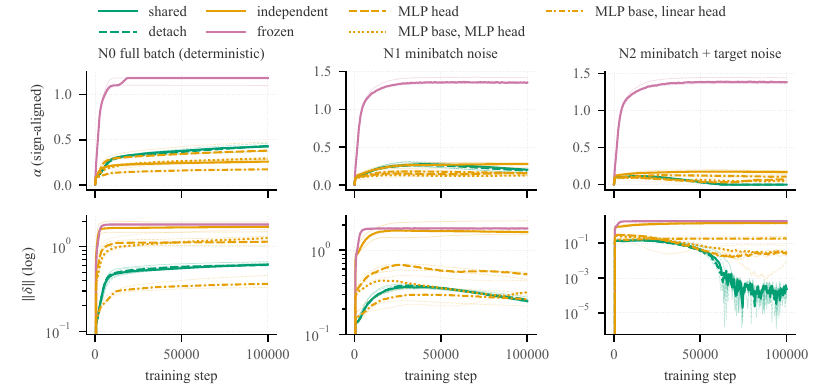}
  \caption{%
    \textbf{Without gradient noise nothing returns.} Minimal model under Adam. Sign-aligned $\alpha$ (top) and $\|\delta\|$ on a log scale (bottom) over 100K steps under the three noise levels; color by whether the base head can reproduce the offset, green where it can, orange where it cannot, pink where the base is frozen, with the line style separating arms inside a color; thin lines are seeds and thick lines the median.
    Under N0 nothing returns; under N1 the shared and detached heads turn over; under N2 they return to zero, the MLP head on the linear base collapses, and the linear head on the MLP base stays.
  }
  \label{fig:toy}
\end{figure}
\FloatBarrier

\FloatBarrier
\section{A nonlinear base head, with and without a linear path}
\label{app:p1bc}

The minimal model separates exact reproduction from approximation with its decisive control, the linear head on the MLP base, which keeps its amplitude at every noise level (Figure~\ref{fig:toy}).
This appendix runs that separation on the agents, on the mechanism suite and the protocol of \S\ref{sec:experiments}: 6 environments, TD3 and SAC, 5 seeds, one million steps, 120 runs.
In both arms the base head is a one-hidden-layer network of 64 units on the shared representation and the shaping head stays linear, so the offset is linear and the base head is not; in the second arm a zero-initialized linear map of $z$ is added to the base head's output.
The two arms are the same function at initialization, seed by seed, and differ only in whether the base head's own parameters can reproduce $\mu + \alpha h$ exactly.
With the linear path they can, by $W_{\mathrm{skip}} \mapsto W_{\mathrm{skip}} + \alpha W_h$ and $b_{\mathrm{skip}} \mapsto b_{\mathrm{skip}} + \alpha b_h$; without it they cannot, since the function class of a one-hidden-layer network is not closed under adding an affine map.

\paragraph{Outcome.}
The head engaged in every cell of both arms.
Without the linear path the median residual over the twelve cells is 0.786 and the arc completes in none of them.
With it the median residual is 0.012 and the arc completes in nine cells.
The residual is lower with the path than without it in twelve of twelve paired cells.
Removal costs an interquartile mean of $+0.296$ [0.240, 0.360] without the path and $-0.0002$ [$-0.0007$, 0.0004] with it.
Table~\ref{tab:p1bc} gives every cell.
Without the path, the means over the six environments quoted in \S\ref{sec:mechanism} are a residual of 0.60 under TD3 and 0.82 under SAC and a removal cost of $+0.38$ and $+0.35$.

\begin{table}[ht]
  \caption{%
    \textbf{A linear path lowers the residual in every cell and the removal cost to within 0.004 of zero.}
    Terminal residual $r$ and removal cost per (environment, algorithm) cell, each the mean over 5 seeds; \emph{base} is the one-hidden-layer base head and $+$\emph{path} the same base head with a zero-initialized linear map added to its output.
    Cost is the normalized return of the shaped actor minus the base actor alone, as in \S\ref{sec:measure}.
  }
  \label{tab:p1bc}
  \centering
  \small\setlength{\tabcolsep}{5pt}
  \begin{tabular}{lrrrrrrrr}
\toprule
 & \multicolumn{4}{c}{Terminal residual $r$} & \multicolumn{4}{c}{Removal cost} \\
\cmidrule(lr){2-5}\cmidrule(lr){6-9}
 & \multicolumn{2}{c}{TD3} & \multicolumn{2}{c}{SAC} & \multicolumn{2}{c}{TD3} & \multicolumn{2}{c}{SAC} \\
\cmidrule(lr){2-3}\cmidrule(lr){4-5}\cmidrule(lr){6-7}\cmidrule(lr){8-9}
Environment & base & $+$path & base & $+$path & base & $+$path & base & $+$path \\
\midrule
Hopper & 0.476 & 0.005 & 0.902 & 0.183 & $+$0.786 & $-$0.004 & $+$0.559 & $-$0.001 \\
HalfCheetah & 0.765 & 0.006 & 0.954 & 0.008 & $+$0.367 & $+$0.003 & $+$0.401 & $+$0.002 \\
Humanoid & 0.915 & 0.081 & 0.765 & 0.016 & $+$0.947 & $-$0.002 & $+$0.572 & $-$0.001 \\
finger/spin & 0.173 & 0.003 & 0.556 & 0.351 & $+$0.012 & 0.000 & $+$0.004 & $-$0.001 \\
walker/run & 0.452 & 0.003 & 0.956 & 0.008 & $+$0.137 & $+$0.002 & $+$0.323 & 0.000 \\
humanoid/walk & 0.829 & 0.266 & 0.806 & 0.097 & $+$0.007 & $-$0.001 & $+$0.218 & $-$0.002 \\
\bottomrule
\end{tabular}

\end{table}

\FloatBarrier
\section{The MLP shaping head at three million steps}
\label{app:p2}

Section~\ref{sec:mechanism} reports that a one-hidden-layer shaping head is not absorbed within one million steps at any of three widths; the alternative reading is that the trunk could reshape itself until the offset is linear in the representation, and one million steps may not be enough for it.
This appendix runs the width-64 head three times as long on the mechanism suite: 6 environments, TD3 and SAC, 3 seeds, three million steps, 36 runs, everything else as in \S\ref{sec:experiments}.

\paragraph{Outcome.}
At one million steps the mean residual over the six environments is 0.59 under TD3 and 0.84 under SAC, against 0.68 and 0.82, the width-64 values of the width sweep of \S\ref{sec:mechanism}, whose widths 16, 64 and 256 give 0.62, 0.68 and 0.58 under TD3 and 0.83, 0.82 and 0.74 under SAC.
At three million steps the median residual over the twelve (environment, algorithm) cells is 0.73 and the arc completes in none.
The median change in residual from one million to three million steps, within the same runs, is a drop of 0.02; in four of the six TD3 cells the residual rises.
Removal costs an interquartile mean of $+0.91$ [0.86, 0.95] of normalized return.
Table~\ref{tab:p2} gives every cell.
Three times the budget leaves the nonlinear head unabsorbed in every cell; with Appendix~\ref{app:p1bc}, which separates the same two readings for a linear offset, the budget reading is closed on both sides.

\begin{table}[ht]
  \caption{%
    \textbf{The MLP shaping head stays unabsorbed at three million steps.}
    Residual $r$ of the same runs read at one million and at three million steps, and removal cost at three million, per (environment, algorithm) cell for the width-64 head, each the mean over 3 seeds.
    Cost is the normalized return of the shaped actor minus the base actor alone, as in \S\ref{sec:measure}.
  }
  \label{tab:p2}
  \centering
  \small\setlength{\tabcolsep}{5pt}
  \begin{tabular}{lrrrrrr}
\toprule
 & \multicolumn{3}{c}{TD3} & \multicolumn{3}{c}{SAC} \\
\cmidrule(lr){2-4}\cmidrule(lr){5-7}
Environment & $r$ 1M & $r$ 3M & cost 3M & $r$ 1M & $r$ 3M & cost 3M \\
\midrule
Hopper & 0.597 & 0.735 & $+$0.820 & 0.798 & 0.721 & $+$0.706 \\
HalfCheetah & 0.700 & 0.829 & $+$1.221 & 0.858 & 0.782 & $+$1.337 \\
Humanoid & 0.851 & 0.848 & $+$0.790 & 0.952 & 0.916 & $+$1.177 \\
finger/spin & 0.382 & 0.467 & $+$0.702 & 0.783 & 0.672 & $+$1.092 \\
walker/run & 0.382 & 0.437 & $+$0.976 & 0.798 & 0.609 & $+$0.948 \\
humanoid/walk & 0.644 & 0.385 & $+$0.001 & 0.856 & 0.863 & $+$0.008 \\
\bottomrule
\end{tabular}

\end{table}

\FloatBarrier
\section{Three million steps for the unfinished TD3 arcs}
\label{app:p4}

Under TD3 the arc completes within one million steps in seven of twenty environments (\S\ref{sec:removal}).
Whether the other thirteen are short of budget cannot be read from a single budget, so this appendix runs six of them three times as long: manipulator/bring\_ball, acrobot/swingup, finger/spin, humanoid/walk, InvertedDoublePendulum and InvertedPendulum, the three highest residuals among the thirteen at one million steps and three from the middle of that ranking, 5 seeds, three million steps, 30 runs, every setting as in \S\ref{sec:experiments}.
Because the exploration and warm-up schedules are absolute step counts, a three-million-step run is a one-million-step run continued, and each run carries its own reading at one million steps as a paired control.

\paragraph{Outcome.}
Every environment engaged in five of five seeds, and the residual at three million steps is below the residual of the same runs at one million steps in six of six environments: 0.827 to 0.732, 0.488 to 0.215, 0.321 to 0.190, 0.206 to 0.176, 0.490 to 0.056 and 0.294 to 0.225.
On InvertedDoublePendulum the ratio falls below the completion threshold while a large offset remains in action space, and removing the head there costs $+0.49$ of normalized return.
Table~\ref{tab:p4} gives every environment.

\begin{table}[ht]
  \caption{%
    \textbf{Three times the budget lowers the residual $r$ in each of six unfinished TD3 arcs.}
    Residual $r$ and removal cost of the same runs read at one million and at three million steps, and the run's peak amplitude $\|\delta\|$ (pre-activation, \S\ref{sec:measure}), each the mean over 5 seeds; cost is the normalized return of the shaped actor minus the base actor alone.
  }
  \label{tab:p4}
  \centering
  \small\setlength{\tabcolsep}{5pt}
  \begin{tabular}{lrrrrr}
\toprule
 & \multicolumn{2}{c}{Residual $r$} & \multicolumn{2}{c}{Removal cost} & Peak \\
\cmidrule(lr){2-3}\cmidrule(lr){4-5}
Environment & 1M & 3M & 1M & 3M & $\|\delta\|$ \\
\midrule
manipulator/bring\_ball & 0.827 & 0.732 & $+$0.935 & $+$0.056 & 20395\phantom{.0} \\
acrobot/swingup & 0.488 & 0.215 & $+$0.033 & $-$0.038 & 22.1 \\
finger/spin & 0.321 & 0.190 & $+$0.155 & $+$0.009 & 160\phantom{.0} \\
humanoid/walk & 0.206 & 0.176 & $+$0.045 & $+$0.004 & 285\phantom{.0} \\
InvertedDoublePendulum & 0.490 & 0.056 & $+$0.950 & $+$0.489 & 21.1 \\
InvertedPendulum & 0.294 & 0.225 & $+$0.376 & $+$0.487 & 5.7 \\
\bottomrule
\end{tabular}

\end{table}

\paragraph{The offset in action space.}
The residual is a ratio to the run's own peak amplitude, taken before the squashing; the cost of removal depends on the offset that remains in action space.
On the two pendulum tasks the peak is 21 and 5.7 pre-activation units on a one-dimensional action, so a residual of 0.056 on InvertedDoublePendulum still leaves a mean offset of 0.85 before the squashing and of 0.18 in an action bounded by one after it, and the base actor without it keeps half the normalized return; on InvertedPendulum a residual of 0.225 still leaves a mean offset of 1.2 before the squashing, and removal costs $+0.49$ there too.
On humanoid/walk the mismatch runs the other way: a pre-activation residual of 0.18 sits on a saturated squashing, the offset applied in action space is 7\% of its peak, and removal costs $+0.004$.
Read in action space, as the mean over states of $\|f(\mu(s) + \delta(s)) - f(\mu(s))\|$, the residual tracks the six removal costs with a rank correlation of 0.71.
The proposition of \S\ref{sec:flat-line} concerns the pre-activation line and the residual is measured there, which is where to read whether the gate has closed; the cost of removal is paid in action space, and where the output map saturates at the peak, predicting that cost requires reading the offset there.

\section{Batch size and the closing time of the gate}
\label{app:p8}

Proposition~\ref{prop:drift} describes the return of the gate as a drift under gradient noise at a rate proportional to the noise level, and Appendix~\ref{app:toy} confirms the proportionality in the idealization.
On the agents the noise level is set by the training batch size $B$, since the variance of a minibatch gradient scales as $1/B$, so the proposition predicts that the gate closes sooner at smaller $B$.
This appendix varies $B$ and nothing else: the shared arm under TD3 and SAC on HalfCheetah, Humanoid and walker/run, the three environments of the mechanism suite whose arc completes within one million steps in at least four of five seeds under both critics at the default $B = 256$, run again at $B \in \{64, 256, 1024\}$ with 5 seeds for one million steps, 90 runs, every other hyperparameter as in \S\ref{sec:experiments}; the batch of evaluation states behind the amplitude is fixed and does not vary with $B$.

\paragraph{Measures.}
Per run, the closing time $T$ is the number of steps from the peak of the amplitude to the first evaluation at which the trailing mean over 20 evaluations is at or below 10\% of the peak, censored at one million steps.
A cell is one (environment, critic, $B$) and takes the median over seeds; a pair is one (environment, critic) across the three batch sizes.

\paragraph{Outcome.}
Every cell engaged in five of five seeds.
The arc completes earlier at $B = 64$ than at $B = 1024$ in five of six pairs; the exception is walker/run under TD3, 535K against 500K steps.
Table~\ref{tab:p8} gives every cell and Figure~\ref{fig:p8} the median curves.

\begin{table}[ht]
  \caption{%
    \textbf{The gate across batch sizes.}
    Closing time $T$ from the peak of the amplitude (K steps), median over 5 seeds; definition in the text.
  }
  \label{tab:p8}
  \centering
  \small\setlength{\tabcolsep}{5pt}
  \begin{tabular}{llrr}
\toprule
Environment & Critic & $B$ & $T$ \\
\midrule
HalfCheetah & TD3 & 64 & 520 \\
 &  & 256 & 490 \\
 &  & 1024 & 705 \\
\addlinespace
HalfCheetah & SAC & 64 & 330 \\
 &  & 256 & 355 \\
 &  & 1024 & 440 \\
\addlinespace
Humanoid & TD3 & 64 & 195 \\
 &  & 256 & 120 \\
 &  & 1024 & 205 \\
\addlinespace
Humanoid & SAC & 64 & 110 \\
 &  & 256 & 110 \\
 &  & 1024 & 135 \\
\addlinespace
walker/run & TD3 & 64 & 535 \\
 &  & 256 & 730 \\
 &  & 1024 & 500 \\
\addlinespace
walker/run & SAC & 64 & 135 \\
 &  & 256 & 270 \\
 &  & 1024 & 345 \\
\bottomrule
\end{tabular}

\end{table}

\begin{figure}[ht]
  \centering
  \includegraphics[width=\textwidth]{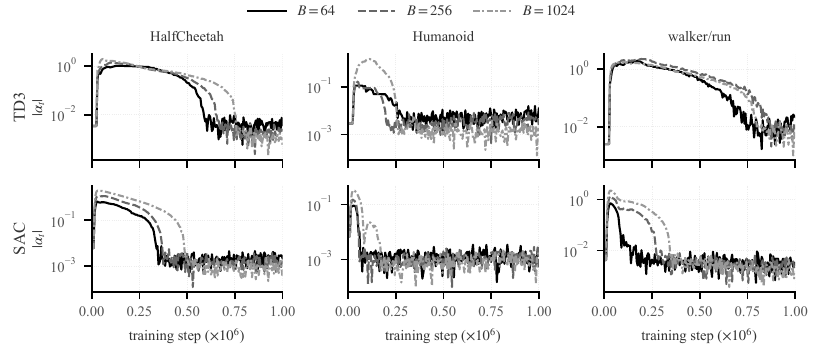}
  \caption{\textbf{The gate rises and falls at every batch size.} $|\alpha_t|$ during training, median over 5 seeds, log scale; rows labelled TD3 and SAC; black $B = 64$, grey dashed $B = 256$, light grey dash-dot $B = 1024$. The closing time $T$ of Table~\ref{tab:p8} is read from the amplitude rather than from these curves.}
  \label{fig:p8}
\end{figure}

\section{The same head composed multiplicatively}
\label{app:p9}

Among the seven arms of \S\ref{sec:arms}, an offset linear in the representation the base reads is inexpressible only where the base head is changed (Appendix~\ref{app:p1bc}); here the base head is that of the shared arm and the linear head enters the action multiplicatively, so that the offset is inexpressible while the head stays linear.
The gain arm acts with $f\bigl(\mu(s) \odot (1 + \alpha\, g(s))\bigr)$, $g(s) = W_g z(s)$ with $W_g$ zero-initialized and no bias, so that the shaped actor coincides with the base actor at initialization exactly as in the shared arm and the head has the same parameter count less one bias vector.
The composite carries the term $\alpha\,(W_\mu z) \odot (W_g z)$, quadratic in $z$ wherever $\alpha W_g \neq 0$ and the base head reads $z$ in a coordinate the gain reads, so no base-head parameters reproduce it: the offset is inexpressible although the head is linear, and the condition of \S\ref{sec:arms} predicts that the gate stays open.
The offset is $\delta = \alpha\, \mu \odot g$, taken before $f$, and the amplitude, residual and removal cost are those of \S\ref{sec:measure}; one more quantity is logged, the state-dependence index $\mathrm{sdi} = \mathbb{E}_s\|g(s) - \bar g\| / \mathbb{E}_s\|g(s)\|$ over the evaluation batch, which is near zero only if the head has degenerated into a constant gain, the one case in which the composite becomes expressible again.
The arm runs on the mechanism suite, six environments under TD3 and SAC with seeds 0 to 2, 36 runs of one million steps, against a shared-arm control of the same 12 cells (36 runs) with the same seeds.

\paragraph{Outcome.}
A cell is one (environment, algorithm); the gain arm is summarized by its seed median and the control by its seed maximum, so that a cell counts only if the gain arm exceeds every control seed.
The head engages in every cell, with the peak amplitude above $10^{-3}$ in at least two seeds. The gain residual exceeds the control's in eleven of the twelve cells and the completion threshold $0.10$ in eleven, and in four cells (HalfCheetah and walker/run under TD3 and under SAC) the gain removal cost exceeds the control's by more than 0.1.
Table~\ref{tab:p9} gives every cell.

\begin{table}[ht]
  \caption{%
    \textbf{A linear head composed multiplicatively stays.}
    Terminal residual $r$, removal cost and state-dependence index of the gain arm against the paired shared arm on the mechanism suite at 1M steps. Gain cells give the median over three seeds (the mean for the index) and control cells the maximum over seeds; the cost is in normalized return.
  }
  \label{tab:p9}
  \centering
  \small\setlength{\tabcolsep}{4pt}
  \begin{tabular}{lrrrrrrrrrr}
\toprule
 & \multicolumn{4}{c}{Terminal residual $r$} & \multicolumn{4}{c}{Removal cost} & \multicolumn{2}{c}{$\mathrm{sdi}$} \\
\cmidrule(lr){2-5}\cmidrule(lr){6-9}\cmidrule(lr){10-11}
 & \multicolumn{2}{c}{TD3} & \multicolumn{2}{c}{SAC} & \multicolumn{2}{c}{TD3} & \multicolumn{2}{c}{SAC} & TD3 & SAC \\
\cmidrule(lr){2-3}\cmidrule(lr){4-5}\cmidrule(lr){6-7}\cmidrule(lr){8-9}
Environment & gain & shared & gain & shared & gain & shared & gain & shared & gain & gain \\
\midrule
Hopper & 0.368 & 0.226 & 0.837 & 0.026 & $+$0.018 & $+$0.011 & $-$0.030 & $+$0.001 & 1.09 & 0.58 \\
HalfCheetah & 0.532 & 0.007 & 0.949 & 0.002 & $+$0.651 & $+$0.002 & $+$0.298 & $+$0.002 & 0.94 & 0.36 \\
Humanoid & 0.819 & 0.020 & 0.926 & 0.013 & $+$0.005 & $+$0.007 & $-$0.014 & $+$0.011 & 0.67 & 0.09 \\
finger/spin & 0.211 & 0.420 & 0.779 & 0.324 & $+$0.099 & $+$0.232 & $+$0.010 & $+$0.014 & 0.74 & 0.59 \\
walker/run & 0.099 & 0.010 & 0.947 & 0.005 & $+$0.584 & $+$0.003 & $+$0.148 & $+$0.001 & 1.02 & 0.30 \\
humanoid/walk & 0.704 & 0.216 & 0.881 & 0.237 & $+$0.008 & 0.000 & $-$0.001 & $+$0.043 & 0.63 & 0.34 \\
\bottomrule
\end{tabular}

\end{table}

\paragraph{Reading.}
In eleven of the twelve cells the gate does not close: the gain residual is above $0.70$ in eight cells and above the control's seed maximum in eleven, and the index stays between $0.30$ and $1.09$ in every cell but Humanoid under SAC ($0.09$), so the head kept a state-dependent gain rather than collapsing to a constant one.
The one cell at the threshold, walker/run under TD3, keeps ten times the control's residual (0.099 against 0.010).
Removing the gain head costs $+0.65$ and $+0.58$ of normalized return on HalfCheetah and walker/run under TD3 and $+0.30$ and $+0.15$ under SAC, the four cells in which the gain head carries a large part of the shaped action; in the other cells the gate stays open at a small amplitude and removal moves the return by at most $0.03$, since the offset there moves the action little (Appendix~\ref{app:p4}), except on finger/spin under TD3, $+0.10$.
The one cell in which the gain residual is below the control's, finger/spin under TD3, is one where the shared arc itself does not complete within the budget (seed maximum $r = 0.42$).

\section{Evaluation protocol and residual conventions}
\label{app:eval}

\paragraph{Evaluation.}
The twenty-environment configurations are trained for 1M environment steps with 5 seeds on each of the 20 continuous-control environments, ten from MuJoCo v5 and ten from the DeepMind Control Suite.
Evaluation is paired and runs every 5{,}000 steps: the same checkpoint is rolled out twice on the same environment seeds, once with the shaping head active and once with it removed, so the removal cost of \S\ref{sec:measure} is a difference between two returns measured at the same point of the same run rather than between two runs.
A terminal score is the mean of the last $W = 20$ evaluations, that is, the last 100K steps of training.
The same window is the numerator of the terminal residual in~\eqref{eq:residual}. The four checkpoints of Table~\ref{tab:checkpoint} are fixed per run from the trajectory of its offset in action space alone, the mean over the evaluation states of $\|f(\mu(s) + \delta(s)) - f(\mu(s))\|$: \emph{rising} is the first evaluation at which it reaches half of its peak, \emph{peak} is the maximum, \emph{post-peak} is the first later evaluation at which it is back at or below half of the peak (midway between the peak and the end of training when there is none), and \emph{final} is the end of training; the saved checkpoint nearest each is rolled out for 10 deterministic episodes with and without the head, and the interval is a 95\% percentile bootstrap over the 30 runs of each algorithm on the mechanism suite.

\paragraph{Normalization.}
Per-environment scores are min-max normalized as $(\text{score} - \text{random}) / (\text{best} - \text{random})$.
The random anchor is the mean return of a uniform random policy (5 seeds, 100 episodes); the best anchor is, per environment, the maximum over evaluated methods of the seed-mean terminal return (last 20 evaluations), so a single run can score above 1; the anchors are fixed per environment and shared by every figure and table.
Min-max normalization with two reference points per task follows the rliable protocol~\citep{agarwal2021rliable}, and since every method shares the same anchors, normalization preserves the sign and order of every within-environment difference.

\paragraph{Aggregation.}
Aggregate scores are the interquartile mean (IQM) with 95\% stratified bootstrap confidence intervals~\citep{agarwal2021rliable} over 50{,}000 replicates.
Where a removal cost is reported as an IQM, it is the IQM of per-run paired differences, not a difference of two IQMs.

\paragraph{Residual conventions.}
We form the terminal residual per run by Equation~\ref{eq:residual}, as the mean of the last 20 evaluations of $\|\delta_t\|$ over that run's own peak, and take an environment's residual to be the mean of $r$ over its seeds.
Three other conventions are possible.
Table~\ref{tab:residual-conventions} recomputes under each of them the two quantities that depend on the residual: the number of environments counted as completed at the fixed threshold $r \leq 0.10$, and the rank correlation between an environment's residual and the size of its removal cost.
The completed count is 7, 8, 8 and 10 under TD3 and 13, 13, 14 and 13 under SAC across the four conventions, and the correlation stays between 0.565 and 0.671 with $p < 0.01$ in every case, so the correlation reported in \S\ref{sec:removal} does not depend on the convention.

\begin{table}[ht]
  \caption{%
    \textbf{Residual conventions.}
    Each row is one way of forming an environment's terminal residual from the per-run amplitude trajectories.
    \emph{Completed}: environments with $r \leq 0.10$, out of 20.
    \emph{Spearman}: rank correlation between an environment's residual and the magnitude of its removal cost, with its $p$-value.
    The first row is the convention used in the paper, defined in~\eqref{eq:residual}; it counts no more environments as completed than any other row, and its correlation is the one reported in \S\ref{sec:removal}.
  }
  \label{tab:residual-conventions}
  \centering
  \small\setlength{\tabcolsep}{3pt}
  \begin{tabular}{@{}>{\raggedright\arraybackslash}p{0.46\textwidth}cccc@{}}
    \toprule
    & \multicolumn{2}{c}{TD3} & \multicolumn{2}{c}{SAC} \\
    \cmidrule(lr){2-3}\cmidrule(lr){4-5}
    Convention & completed & Spearman ($p$) & completed & Spearman ($p$) \\
    \midrule
    Per-run $r$, mean over seeds
      & 7 & 0.614 (0.0040) & 13 & 0.636 (0.0026) \\
    Seed-mean trajectory, last 20 evaluations over peak
      & 8 & 0.608 (0.0045) & 13 & 0.608 (0.0045) \\
    Seed-mean trajectory, last point over peak
      & 8 & 0.565 (0.0094) & 14 & 0.571 (0.0085) \\
    Per-run $r$, median over seeds
      & 10 & 0.626 (0.0032) & 13 & 0.671 (0.0012) \\
    \bottomrule
  \end{tabular}
\end{table}

\section{Per-environment results}
\label{app:per-env}

Table~\ref{tab:per-env} lists the terminal IQM of the shaped actors and the frozen arms in every environment, and Figure~\ref{fig:removal} the per-environment removal cost with its interval, pooled over the shaped TD3 and SAC actors (17 of 20 intervals contain zero for SAC alone, 15 for TD3 alone).

\begin{table}[htbp]
  \caption{\textbf{Per-environment terminal IQM of the shaped actors} (20 environments, 5 seeds, shaped evaluation, min-max normalized); $\pm$ is the half-width of the 95\% stratified bootstrap CI. AS-TD3-Fr and AS-SAC-Fr are the frozen arms.}
  \label{tab:per-env}
  \centering
  \fontsize{8}{9.5}\selectfont\setlength{\tabcolsep}{4.5pt}
  \begin{tabular}{@{}lrrrrr@{}}
\toprule
Environment & \multicolumn{1}{c}{AS-TD3} & \multicolumn{1}{c}{AS-TD3-Fr} & \multicolumn{1}{c}{AS-SAC} & \multicolumn{1}{c}{AS-SAC-Fr} & \multicolumn{1}{c}{AS-PPO} \\
\midrule
InvertedPendulum & 1.009 $\pm$ 0.006 & 0.912 $\pm$ 0.109 & 1.004 $\pm$ 0.024 & 1.000 $\pm$ 0.015 & 1.018 $\pm$ 0.030 \\
InvertedDoublePendulum & 1.001 $\pm$ 0.012 & 0.980 $\pm$ 0.032 & 1.005 $\pm$ 0.022 & 0.036 $\pm$ 0.003 & 0.763 $\pm$ 0.127 \\
Reacher & 1.004 $\pm$ 0.005 & 0.950 $\pm$ 0.007 & 1.020 $\pm$ 0.004 & 0.588 $\pm$ 0.183 & 0.969 $\pm$ 0.008 \\
Swimmer & 0.805 $\pm$ 0.122 & 0.594 $\pm$ 0.031 & 0.909 $\pm$ 0.144 & 0.649 $\pm$ 0.055 & 0.873 $\pm$ 0.397 \\
Hopper & 0.971 $\pm$ 0.097 & 0.050 $\pm$ 0.008 & 1.017 $\pm$ 0.139 & 0.961 $\pm$ 0.121 & 1.057 $\pm$ 0.214 \\
HalfCheetah & 0.995 $\pm$ 0.128 & 0.461 $\pm$ 0.053 & 1.023 $\pm$ 0.020 & 0.425 $\pm$ 0.073 & 0.568 $\pm$ 0.059 \\
Walker2d & 0.889 $\pm$ 0.100 & 0.083 $\pm$ 0.182 & 0.899 $\pm$ 0.045 & 0.310 $\pm$ 0.158 & 0.963 $\pm$ 0.347 \\
Ant & 1.034 $\pm$ 0.203 & 0.702 $\pm$ 0.181 & 0.780 $\pm$ 0.115 & 0.013 $\pm$ 0.001 & 0.353 $\pm$ 0.054 \\
Humanoid & 0.941 $\pm$ 0.014 & 0.089 $\pm$ 0.037 & 0.964 $\pm$ 0.127 & $-$0.010 $\pm$ 0.000 & 0.192 $\pm$ 0.322 \\
Pusher & 0.984 $\pm$ 0.015 & 0.977 $\pm$ 0.012 & 1.019 $\pm$ 0.005 & 0.900 $\pm$ 0.013 & 0.901 $\pm$ 0.017 \\
acrobot/swingup & 0.127 $\pm$ 0.095 & 0.039 $\pm$ 0.087 & 0.284 $\pm$ 0.118 & 0.043 $\pm$ 0.071 & 1.101 $\pm$ 0.371 \\
cartpole/swingup & 0.992 $\pm$ 0.036 & 0.870 $\pm$ 0.305 & 1.006 $\pm$ 0.013 & 0.992 $\pm$ 0.030 & 0.964 $\pm$ 0.025 \\
cheetah/run & 1.058 $\pm$ 0.145 & 0.319 $\pm$ 0.265 & 0.942 $\pm$ 0.079 & 0.578 $\pm$ 0.065 & 0.618 $\pm$ 0.225 \\
hopper/hop & 0.857 $\pm$ 0.831 & 0.011 $\pm$ 0.007 & 0.837 $\pm$ 0.329 & 0.017 $\pm$ 0.180 & 0.019 $\pm$ 0.002 \\
walker/run & 0.895 $\pm$ 0.157 & 0.029 $\pm$ 0.056 & 1.037 $\pm$ 0.087 & 0.182 $\pm$ 0.034 & 0.103 $\pm$ 0.050 \\
humanoid/walk & 0.018 $\pm$ 0.440 & 0.007 $\pm$ 0.001 & 0.027 $\pm$ 0.374 & 0.007 $\pm$ 0.001 & 0.022 $\pm$ 0.001 \\
swimmer/swimmer15 & 1.074 $\pm$ 0.275 & 0.151 $\pm$ 1.008 & 0.835 $\pm$ 0.439 & 0.502 $\pm$ 0.254 & 0.429 $\pm$ 0.181 \\
finger/spin & 1.104 $\pm$ 0.167 & 0.011 $\pm$ 0.230 & 0.924 $\pm$ 0.220 & 0.547 $\pm$ 0.155 & 0.018 $\pm$ 0.001 \\
reacher/hard & 0.998 $\pm$ 0.016 & 0.395 $\pm$ 0.161 & 1.004 $\pm$ 0.019 & 0.507 $\pm$ 0.261 & 0.504 $\pm$ 0.063 \\
manipulator/bring\_ball & $-$0.086 $\pm$ 0.268 & $-$0.053 $\pm$ 0.262 & 0.110 $\pm$ 1.883 & $-$0.132 $\pm$ 0.292 & 0.143 $\pm$ 0.673 \\
\bottomrule
\end{tabular}
\end{table}

\begin{figure}[htbp]
  \centering
  \includegraphics[width=\textwidth]{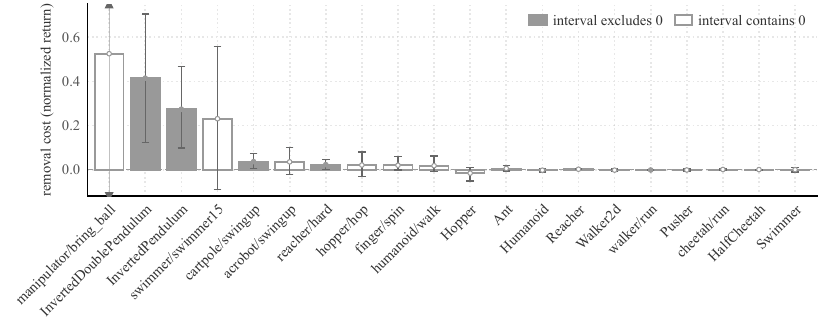}
  \caption{%
    \textbf{The removal-cost interval contains zero in 15 of the 20 environments.} Per-environment removal cost (shaped minus base-only): the mean over the pooled runs with its 95\% bootstrap CI.
    The shaped TD3 and SAC actors pooled, 5 seeds each, min-max normalized, sorted by magnitude. The interval of manipulator/bring\_ball, $[-0.20, +1.58]$, is clipped to the axis.
  }
  \label{fig:removal}
\end{figure}
\FloatBarrier

\section{Comparison with the host algorithm}
\label{app:host}

The comparison of \S\ref{sec:benefit} trains TD3 and the shaped TD3 actor with the same implementation, the hyperparameters of Table~\ref{tab:hyperparameters} and the evaluation protocol of Appendix~\ref{app:eval}, for 1M steps on each of the 20 environments with seeds 0, 1 and 2; the shaped actor carries the head and gate of \S\ref{sec:shaped-actor} with $\varepsilon = 0.01$ and no per-task setting.
Terminal scores are normalized with the anchors used throughout and aggregated by the IQM over the 60 environment-seed pairs; the interval is a 95\% bootstrap of the difference of the two IQMs that resamples seeds within each environment, with the same seeds for both actors (50{,}000 resamples).
The IQM of the shaped actor exceeds that of TD3 by $+0.077$ [0.040, 0.116], and its mean over seeds is higher in 14 of the 20 environments.
With $\varepsilon = 0$, so that neither the gate nor the head receives a gradient (Appendix~\ref{app:eps}), the shaped actor scores as TD3 does on the mechanism suite (a difference of $+0.003$ [$-0.049$, 0.055], 3 seeds): the gain is the head's.

\section{Amplitude decomposition}
\label{app:decomposition}

The amplitude factorizes as $\|\delta\| = |\alpha| \cdot \|h\|$, and Figure~\ref{fig:decomposition} plots the two factors separately for the six environments of the mechanism suite under the default configuration.
Across the twelve (environment, algorithm) cells the median decline of $|\alpha|$ from its own peak is 98.14\%.
The head's output norm does not return to zero: the terminal $\|h\|$ lies between 20\% and 94\% of its own peak, between 20\% and 88\% under TD3 and between 51\% and 94\% under SAC, and above half of its peak in ten of the twelve cells, the two below half being Hopper (0.20) and walker/run (0.43) under TD3.
The two factors are therefore not symmetric, the gate closing almost completely in ten of the twelve cells while the head keeps most of its scale, which is the asymmetry described in \S\ref{sec:flat-line}.

\begin{figure}[ht]
  \centering
  \includegraphics[width=\textwidth]{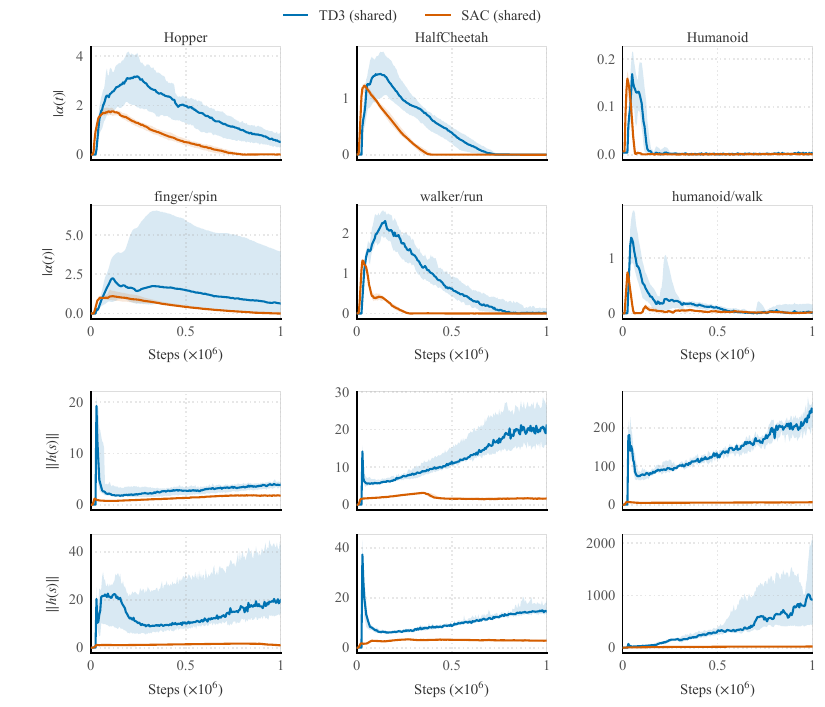}
  \caption{%
    \textbf{Amplitude decomposition: the fall of the amplitude comes from the gate} (6 environments, 5 seeds, default configuration).
    Upper block: $|\alpha|$ over training.
    Lower block: $\|h(s)\|$ over training.
    TD3 in blue, SAC in orange; lines are the interquartile mean over seeds (median of 2{,}000 bootstrap resamples) and shading is the 95\% stratified bootstrap CI.
  }
  \label{fig:decomposition}
\end{figure}

\section{Update-to-data ratio}
\label{app:utd}

Section~\ref{sec:arc} compares the twin critic with an ensemble critic~\citep{chen2021randomized} on the four environments of this experiment, which varies the update-to-data ratio inside the ensemble critic, with $N = 10$ networks and $M = 2$ sampled per update.
Table~\ref{tab:utd} gives the per-environment and aggregate IQM for all four configurations at each ratio; under SAC-REDQ the shaped actor's aggregate IQM is above its host's at every ratio, by 0.029, 0.038 and 0.064.

\begin{table}[ht]
  \caption{%
    \textbf{Update-to-data scaling, per-environment IQM} (4 environments, 5 seeds, shaped evaluation).
    Min-max normalized; $\pm$ half-width of the 95\% stratified bootstrap CI.
    Last row of each group: aggregate IQM.
    All four configurations use the same ensemble critic ($N = 10$, $M = 2$).
  }
  \label{tab:utd}
  \centering
  \small
  \begin{tabular}{lcccc}
\toprule
Environment & TD3-REDQ & AS-TD3-REDQ & SAC-REDQ & AS-SAC-REDQ \\
\midrule
\multicolumn{5}{l}{\textit{UTD = 1}} \\
\midrule
Hopper & 0.983 $\pm$ 0.087 & 1.007 $\pm$ 0.083 & 0.900 $\pm$ 0.234 & 0.827 $\pm$ 0.146 \\
HalfCheetah & 0.962 $\pm$ 0.099 & 0.994 $\pm$ 0.039 & 1.058 $\pm$ 0.044 & 1.090 $\pm$ 0.035 \\
Humanoid & 0.963 $\pm$ 0.022 & 0.962 $\pm$ 0.021 & 1.017 $\pm$ 0.033 & 1.048 $\pm$ 0.064 \\
walker/run & 0.911 $\pm$ 0.178 & 0.814 $\pm$ 0.152 & 0.690 $\pm$ 0.215 & 0.928 $\pm$ 0.165 \\
\midrule
\textbf{Aggregate} & 0.963 $\pm$ 0.037 & 0.962 $\pm$ 0.020 & 0.970 $\pm$ 0.058 & 0.999 $\pm$ 0.052 \\
\midrule
\multicolumn{5}{l}{\textit{UTD = 5}} \\
\midrule
Hopper & 0.975 $\pm$ 0.048 & 1.003 $\pm$ 0.163 & 0.842 $\pm$ 0.203 & 0.814 $\pm$ 0.205 \\
HalfCheetah & 1.067 $\pm$ 0.056 & 1.072 $\pm$ 0.097 & 1.188 $\pm$ 0.124 & 1.219 $\pm$ 0.073 \\
Humanoid & 0.946 $\pm$ 0.026 & 0.952 $\pm$ 0.068 & 1.027 $\pm$ 0.037 & 1.054 $\pm$ 0.076 \\
walker/run & 0.913 $\pm$ 0.127 & 0.911 $\pm$ 0.142 & 0.865 $\pm$ 0.034 & 0.872 $\pm$ 0.072 \\
\midrule
\textbf{Aggregate} & 0.976 $\pm$ 0.025 & 0.978 $\pm$ 0.052 & 0.966 $\pm$ 0.036 & 1.004 $\pm$ 0.051 \\
\midrule
\multicolumn{5}{l}{\textit{UTD = 20}} \\
\midrule
Hopper & 0.856 $\pm$ 0.254 & 1.049 $\pm$ 0.128 & 0.639 $\pm$ 0.153 & 0.733 $\pm$ 0.242 \\
HalfCheetah & 1.095 $\pm$ 0.079 & 1.040 $\pm$ 0.219 & 1.114 $\pm$ 0.172 & 1.202 $\pm$ 0.188 \\
Humanoid & 0.945 $\pm$ 0.327 & 0.980 $\pm$ 0.022 & 1.010 $\pm$ 0.020 & 1.054 $\pm$ 0.046 \\
walker/run & 0.957 $\pm$ 0.074 & 0.927 $\pm$ 0.171 & 0.900 $\pm$ 0.126 & 0.952 $\pm$ 0.040 \\
\midrule
\textbf{Aggregate} & 0.974 $\pm$ 0.052 & 0.995 $\pm$ 0.062 & 0.933 $\pm$ 0.067 & 0.997 $\pm$ 0.033 \\
\bottomrule
\end{tabular}
\end{table}

\paragraph{The gate under the two critics.}
Table~\ref{tab:utd-alpha} compares, on these four environments, the terminal-to-peak ratio of $|\alpha|$ under the twin critic of the main experiments and under the ensemble critic at an update-to-data ratio of 1, with 5 seeds each.
Under TD3 the ensemble critic brings the one cell in which the twin critic leaves the ratio above 0.10 at 1M steps, Hopper, from 0.2112 to 0.0022, and the mean ratio over the four environments from 0.0637 to 0.0113; under SAC every ratio is at most 0.0102 with either critic.

\begin{table}[ht]
  \caption{%
    \textbf{Terminal-to-peak ratio of $|\alpha|$ under the twin and the ensemble critic} (4 environments, 5 seeds, shaped actors).
    Per environment, the mean over seeds of $|\alpha_t|$ is averaged over the last 20 evaluations and divided by its maximum over training; \emph{twin} is the critic of the main experiments and \emph{ensemble} the critic of Table~\ref{tab:utd} ($N = 10$, $M = 2$) at an update-to-data ratio of 1. The ensemble critic brings the one TD3 ratio that the twin critic leaves above 0.10, Hopper, to 0.0022.
  }
  \label{tab:utd-alpha}
  \centering
  \small\setlength{\tabcolsep}{5pt}
  \begin{tabular}{lrrrr}
    \toprule
    & \multicolumn{2}{c}{TD3} & \multicolumn{2}{c}{SAC} \\
    \cmidrule(lr){2-3}\cmidrule(lr){4-5}
    Environment & twin & ensemble & twin & ensemble \\
    \midrule
    Hopper & 0.2112 & 0.0022 & 0.0087 & 0.0044 \\
    HalfCheetah & 0.0016 & 0.0087 & 0.0015 & 0.0019 \\
    Humanoid & 0.0185 & 0.0098 & 0.0102 & 0.0076 \\
    walker/run & 0.0235 & 0.0246 & 0.0019 & 0.0038 \\
    \midrule
    Mean & 0.0637 & 0.0113 & 0.0056 & 0.0044 \\
    \bottomrule
  \end{tabular}
\end{table}

\section{Sensitivity to $\varepsilon$}
\label{app:eps}

The initial gate $\alpha \sim U(-\varepsilon,\, \varepsilon)$ is the only constant the method adds; since the head starts at zero, it sets not the offset at initialization, which is zero for every $\varepsilon$, but the learning signal the head receives and the scale at which the head's output first enters the action.
At $\varepsilon = 0$ the offset is identically zero and stays there: the gradient of the actor loss with respect to $(W_h, b_h)$ carries a factor $\alpha$ (\S\ref{sec:flat-line}), so at $\alpha = 0$ the shaping head receives no learning signal.
\citet{khemais2026surgery} shows that a zero gate over a zero-initialized output projection is an exact saddle point that gradient descent cannot escape, which is why the gate is initialized away from zero.
Figure~\ref{fig:eps} compares $\varepsilon \in \{0.01,\, 0.1\}$ on four environments under both algorithms.
A tenfold larger $\varepsilon$ raises the terminal return in four of the eight panels and lowers it in four, by at most 11.4\% (last 20 evaluations, mean over seeds), so within this range the one constant the method adds needs no tuning.

\begin{figure}[ht]
  \centering
  \includegraphics[width=\textwidth]{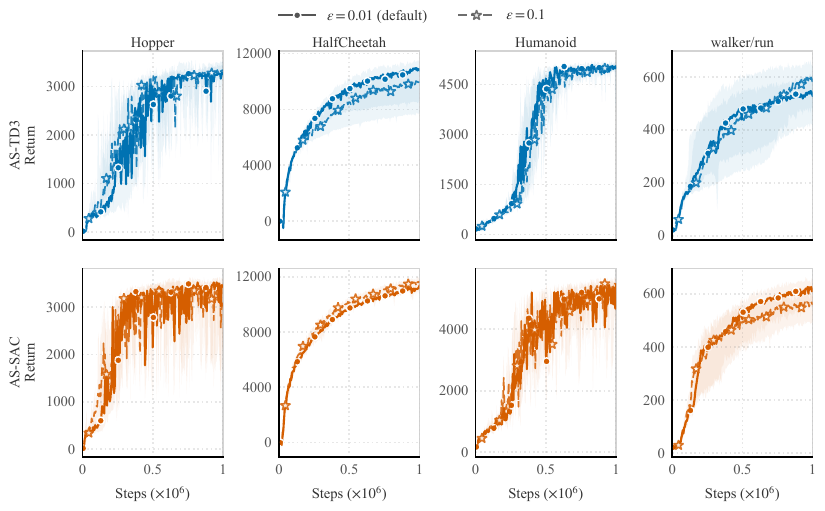}
  \caption{%
    \textbf{The one constant the method adds needs no tuning across a tenfold range of $\varepsilon$} (4 environments, 5 seeds, shaped evaluation, raw return).
    Rows: shaped TD3 (top) and shaped SAC (bottom).
    $\varepsilon = 0.01$ (solid, circles): the value used everywhere else in the paper.
    $\varepsilon = 0.1$ (dashed, stars): a larger initial gate.
    Lines are the interquartile mean over seeds (median of 2{,}000 bootstrap resamples) and shading is the 95\% stratified bootstrap CI.
  }
  \label{fig:eps}
\end{figure}

\section{Atari}
\label{app:atari}

The three interquartile means of \S\ref{sec:discrete}, 0.183 [0.157, 0.216] for the base actor alone, 0.170 for the shaped actor and 0.154 for the unshaped host, and the difference of 0.013 [0.003, 0.023] by which the base actor alone exceeds the shaped actor (a difference of interquartile means, resampled with matched seeds), all come from a single evaluation: at the 1M-step checkpoint, each policy acts deterministically (argmax) for 10 episodes on matched seeds, and the interquartile mean is taken over the 130 seed-game pairs, with a game-stratified, seed-matched bootstrap for the intervals.

Table~\ref{tab:atari-arc} gives the amplitude ladder of the seven games in which the gate engages, measured at the saved checkpoints of each run (a median of 14): three complete the arc within 1M steps (Seaquest, Alien, Boxing; $r$ from 0.016 to 0.057) and four are still active at 1M (BattleZone, ChopperCommand, Pong, Kangaroo; $r$ from 0.34 to 1.0).
In the other 19 games, Amidar, Assault, Asterix, BankHeist, Breakout, CrazyClimber, DemonAttack, Freeway, Frostbite, Gopher, Hero, Jamesbond, Krull, KungFuMaster, MsPacman, PrivateEye, Qbert, RoadRunner, UpNDown%
, the peak of $\|\delta_t\|$ never reaches $10^{-3}$ (median peak $3.3\times10^{-5}$), so no arc exists there to complete.

\begin{table}[ht]
  \caption{%
    \textbf{Atari amplitude ladder for the seven games in which the gate engages.}
    Seed-averaged $\|\delta_t\|$, the expected probability-space L2 distance between the shaped and the base policy on a fixed buffer of 512 frames per run, at the saved checkpoint nearest each share of the 1M-step budget (5 seeds per game); cells below $10^{-4}$ are shown as $<$0.0001.
    $r$ is the terminal amplitude divided by the peak over all saved checkpoints of the run (a median of 14), computed from unrounded values; the arc is completed when $r \le 0.10$.
  }
  \label{tab:atari-arc}
  \centering
  \small\setlength{\tabcolsep}{4pt}
  \begin{tabular}{lrrrrrrll}
\toprule
 & \multicolumn{6}{c}{$\|\delta_t\|$ at the checkpoint nearest a share of the 1M-step budget} & & \\
\cmidrule(lr){2-7}
Game & 5\% & 10\% & 25\% & 50\% & 75\% & 100\% & $r$ & Arc \\
\midrule
Seaquest & 0.0003 & 0.0035 & 0.0007 & 0.0002 & $<$0.0001 & 0.0001 & 0.016 & completed \\
Alien & 0.0002 & 0.0019 & $<$0.0001 & $<$0.0001 & $<$0.0001 & 0.0001 & 0.030 & completed \\
Boxing & 0.0001 & 0.0011 & $<$0.0001 & $<$0.0001 & $<$0.0001 & 0.0001 & 0.057 & completed \\
\midrule
BattleZone & 0.0003 & 0.0006 & 0.0149 & 0.0254 & 0.0206 & 0.0086 & 0.34 & active \\
ChopperCommand & 0.0002 & 0.0018 & 0.0477 & 0.0349 & 0.0328 & 0.0291 & 0.49 & active \\
Pong & 0.0005 & 0.0008 & 0.0176 & 0.0334 & 0.0882 & 0.0838 & 0.95 & active \\
Kangaroo & 0.0002 & 0.0014 & 0.0042 & 0.0010 & 0.0020 & 0.0099 & 1.00 & active \\
\bottomrule
\end{tabular}

\end{table}

\FloatBarrier
\section{Extended related work}
\label{app:related}

This appendix gives the longer form of \S\ref{sec:related}.

\paragraph{Reward shaping and the action channel.}
Potential-based shaping characterizes the reward modifications that leave the optimal policy unchanged and can therefore be removed at deployment~\citep{ng1999reward}; the characterization is independent of the learner.
The action channel has no invariance of that kind, since an offset on the action changes the state distribution, so the same question has to be answered by the training dynamics rather than by the MDP.
Our answer is a condition for absorption and the mechanism behind it. Guidance can also enter through the replay buffer: injecting transitions from a model-predictive controller steers TD3 and SAC toward a designer-preferred gait while the reward stays unchanged~\citep{xing2026mpcinjection}; that guidance lives in the data, so nothing remains in the policy to remove.
The phrase action space shaping has been used for reducing a discrete action space before training~\citep{kanervisto2020actionspace}, and reshaping the action space during training has meant learning limits on the allowed actions~\citep{pham2018ceres}; both are unrelated to the offset studied here.

\paragraph{Folding a training-time branch into a layer.}
Structural re-parameterization trains a multi-branch block and folds it algebraically into a single layer for inference~\citep{ding2021repvgg,hu2022orepa,dlr2026}, and linear over-parameterization does the same for stacked linear layers~\citep{guo2020expandnets,arora2018implicit,du2018balanced}; the merge is performed by hand after training.
Our linear head admits the same merge (\S\ref{sec:arms}); we ask whether training itself arrives at the merged solution, and find that it does only under the conditions of \S\ref{sec:mechanism}.

\paragraph{Learnable residual scales.}
In ReZero~\citep{bachlechner2021rezero} and LayerScale~\citep{touvron2021cait} a residual branch is scaled by a learnable gate that starts at or near zero; the gate grows and stays because the branch is a permanent part of the model, whereas on an actor the same gate rises and then returns to zero, and this paper explains why.

\paragraph{Train with, deploy without.}
Removing a training-time component at deployment is established for the observation channel by privileged-sensor scaffolds~\citep{freed2024scaffolder}.
Residual distillation trains a plain network alongside its shortcut-equipped twin and removes the shortcuts by a schedule~\citep{li2020residualdistillation}; gated normalization is tapered out by a schedule~\citep{tapernorm2026}; action priors added to the sampled action decay by a schedule~\citep{sood2024decap,sood2025apex}; jump-start and attenuated residual methods phase out a base or guide policy by a schedule~\citep{uchendu2023jsrl,trumpp2026arpo,bolychev2026agency}, as does model-predictive control used as scaffolding for real-world dexterous manipulation~\citep{kucuktabak2026mpcscaffold}.
The same design has appeared for language agents, where skill, critique, or budget scaffolds are supplied during reinforcement learning and withdrawn by a curriculum before deployment~\citep{lu2026skill0,he2026siri,shi2026pats,sun2026anysearch}, and where an agent harness used as guidance during training is distilled into the weights by fine-tuning and removed at deployment~\citep{ye2026harnesszero}.
In every case a designer decides when the component leaves.
The condition of \S\ref{sec:theory} says where a schedule is needed and where it is not: an expressible branch on a trainable base drifts toward absorption on its own.

\paragraph{Residual and correction policies.}
Residual policy learning adds a learned correction to a fixed or slowly adapting base~\citep{silver2018residual,johannink2019residual,gou2023racemop,feng2024tsac,wang2025rpg,wei2025pld,wang2026r2po,resfit2025,beyondresiduals2026,warprl2026,muttaqien2026orpa,zhang2026dlsrl,wu2026gatedresidual,ma2026dawn}; the correction is meant to stay and the base is usually frozen, the regime our frozen arm reproduces, in which the offset stays.

\paragraph{Gates that close and branches that fall silent.}
Learnable gates on auxiliary pathways into a frozen backbone tend to close at zero through a dead-gradient regime~\citep{farazi2026gate}; deep attention sublayers learn to mute themselves during pre-training~\citep{saikumar2025datafree}; sparsity penalties prune branches on purpose~\citep{huang2018sss}.
Our $\alpha$ closes with no penalty and only after it has opened, and the three states of \S\ref{sec:measure} keep those cases apart.

\paragraph{Internalization and removability.}
\citet{tsilivis2026internalization} define internalization as a network absorbing an explicit procedure into weights of the same hypothesis class; their chain-of-thought scaffold is removed by a curriculum, ours is absorbed by training, and we add the trainable-base condition, the flat-line mechanism, and a diagnostic read from the amplitude alone.
Concurrent work on annealable soft priors finds that removability depends on the training trajectory~\citep{yang2026removability}, as our checkpoint interventions find within a run for a parametric branch.

\paragraph{Why a redundant parameter drifts to zero.}
Near a manifold of minimizers, gradient noise moves stochastic gradient descent toward flatter points~\citep{blanc2020implicit,damian2021labelnoise,li2022zeroloss}, toward sparse solutions~\citep{pesme2021diagonal}, and onto simpler invariant subnetworks~\citep{kunin2021neuralmechanics,chen2023stochasticcollapse}; Section~\ref{sec:theory} shows that an expressible offset creates such a manifold, flat in $\alpha$ with its flattest point at $\alpha = 0$, and computes the drift on it.

\FloatBarrier
\section{Hyperparameters}
\label{app:hyperparameters}

The twenty-environment runs of \S\ref{sec:experiments} use the settings listed here; the other experiment families deviate only as Table~\ref{tab:hyperparameters-varying} states (Table~\ref{tab:hyperparameters-atari} gives the Atari settings), apart from the head and base variants of \S\ref{sec:arms} and Appendices~\ref{app:p1bc} and~\ref{app:p9}, described where they are used.

\begin{table}[ht]
  \caption{%
    \textbf{Hyperparameters for continuous control.}
    One value per algorithm unless a row names the experiments it applies to.
  }
  \label{tab:hyperparameters}
  \centering
  \small\setlength{\tabcolsep}{3pt}\renewcommand{\arraystretch}{1.2}
  \begin{tabular}{@{}>{\raggedright\arraybackslash}p{0.30\textwidth}>{\raggedright\arraybackslash}p{0.215\textwidth}>{\raggedright\arraybackslash}p{0.215\textwidth}>{\raggedright\arraybackslash}p{0.215\textwidth}@{}}
    \toprule
    & TD3 & SAC & PPO \\
    \midrule
    Actor and critic hidden layers & \multicolumn{3}{>{\centering\arraybackslash}p{0.67\textwidth}@{}}{$2 \times 256$, ReLU} \\
    Actor output & $\tanh$ & $\tanh$-squashed Gaussian & Gaussian, state-independent $\log\sigma$ \\
    Optimizer & \multicolumn{3}{>{\centering\arraybackslash}p{0.67\textwidth}@{}}{Adam} \\
    Learning rate & $3\times10^{-4}$ (actor and critic) & actor, critic, temperature $3\times10^{-4}$ & $3\times10^{-4}$, linearly annealed \\
    Discount $\gamma$ & \multicolumn{3}{>{\centering\arraybackslash}p{0.67\textwidth}@{}}{0.99} \\
    Target update & Polyak $\tau = 0.005$, actor and targets every critic step & Polyak $\tau = 0.005$ & n/a\\
    Batch size & 256 & 256 & rollout 2048, minibatch 64, 10 epochs \\
    Exploration & $\mathcal{N}(0, 0.1^2)$ action noise & temperature tuned to entropy $-|\mathcal{A}|$ & entropy coef.\ 0 \\
    Target smoothing & noise 0.2, clipped at 0.5 & n/a& n/a\\
    Other & n/a& $\log\sigma \in [-20, 2]$ & clip 0.2, GAE $\lambda = 0.95$, value coef.\ 0.5, grad.\ norm 0.5 \\
    Updates per environment step & 1 & 1 & n/a\\
    Replay buffer & $10^6$ & $10^6$ & n/a\\
    Random-action warm-up & 25{,}000 steps & 10{,}000 steps & n/a\\
    Training steps & \multicolumn{3}{>{\centering\arraybackslash}p{0.67\textwidth}@{}}{$10^6$; $3 \times 10^6$ in Appendices~\ref{app:p2} and~\ref{app:p4}} \\
    Evaluation & \multicolumn{3}{>{\centering\arraybackslash}p{0.67\textwidth}@{}}{every 5{,}000 steps, 10 episodes, deterministic policy} \\
    Seeds & \multicolumn{3}{>{\centering\arraybackslash}p{0.67\textwidth}@{}}{5; 3 in Appendices~\ref{app:p2}, \ref{app:p9} and~\ref{app:host}} \\
    \midrule
    \multicolumn{4}{@{}l}{\emph{Shaping head}} \\
    Head $h$ & \multicolumn{3}{>{\raggedright\arraybackslash}p{0.67\textwidth}@{}}{linear map of the actor's last hidden layer (256 units), weights and bias initialized to zero} \\
    MLP head & \multicolumn{3}{>{\raggedright\arraybackslash}p{0.67\textwidth}@{}}{one hidden layer (ReLU), output layer initialized to zero; hidden units 16, 64, 256 (\S\ref{sec:mechanism}); 64 (Appendix~\ref{app:p2})} \\
    Gate $\alpha$ & \multicolumn{3}{>{\raggedright\arraybackslash}p{0.67\textwidth}@{}}{one learnable scalar, $\alpha_0 \sim U(-\varepsilon, \varepsilon)$, $\varepsilon = 0.01$; trained by the actor optimizer} \\
    Action & \multicolumn{3}{>{\raggedright\arraybackslash}p{0.67\textwidth}@{}}{$a = f\big(\mu(s) + \alpha\, h(s)\big)$, $f$ the host's output map ($\tanh$ and scaling for TD3 and SAC)} \\
    \bottomrule
  \end{tabular}
\end{table}

\begin{table}[ht]
  \caption{%
    \textbf{Hyperparameters that differ between experiments.}
    Every other setting is as in Table~\ref{tab:hyperparameters}.
  }
  \label{tab:hyperparameters-varying}
  \centering
  \small\setlength{\tabcolsep}{3pt}\renewcommand{\arraystretch}{1.2}
  \begin{tabular}{@{}>{\raggedright\arraybackslash}p{0.30\textwidth}>{\raggedright\arraybackslash}p{0.33\textwidth}>{\raggedright\arraybackslash}p{0.33\textwidth}@{}}
    \toprule
    Experiment & Setting & Value \\
    \midrule
    Batch size (Appendix~\ref{app:p8}) & batch size & 64, 256, 1024 \\
    Update-to-data ratio (Appendix~\ref{app:utd}) & updates per environment step & 1, 5, 20; 10 critics, minimum over 2 \\
    Gate initialization (Appendix~\ref{app:eps}) & $\varepsilon$ & 0.01, 0.1 \\
    \bottomrule
  \end{tabular}
\end{table}

\begin{table}[ht]
  \caption{%
    \textbf{Hyperparameters for Atari} (discrete SAC).
  }
  \label{tab:hyperparameters-atari}
  \centering
  \small\setlength{\tabcolsep}{3pt}
  \begin{tabular}{@{}>{\raggedright\arraybackslash}p{0.30\textwidth}>{\raggedright\arraybackslash}p{0.68\textwidth}@{}}
    \toprule
    Setting & Value \\
    \midrule
    Preprocessing & frame skip 4, $84 \times 84$ grayscale, 4 stacked frames, up to 30 no-ops \\
    Episode end & life loss in training, game over in evaluation \\
    Network & convolutions $32{\times}8{\times}8/4$, $64{\times}4{\times}4/2$, $64{\times}3{\times}3/1$, fully connected 512 \\
    Optimizer & Adam with numerical constant $10^{-4}$; policy, $Q$ and temperature learning rates $3\times10^{-4}$ \\
    Discount $\gamma$ & 0.99 \\
    Batch size & 64 \\
    Replay buffer & $10^5$ \\
    Learning starts & 20{,}000 steps \\
    Update frequency & every 4 steps \\
    Target network & hard update every 8{,}000 steps \\
    Temperature & automatically tuned to $0.89 \times$ the maximum entropy \\
    Training steps & $10^6$ \\
    Evaluation & 10 deterministic episodes per seed at the 1M-step checkpoint (Appendix~\ref{app:atari}) \\
    Seeds & 5 per game \\
    Shaping head & linear map of the shared network's output to the logits, weights and bias initialized to zero; gate as in Table~\ref{tab:hyperparameters} \\
    \bottomrule
  \end{tabular}
\end{table}

\end{document}